%% file: main.tex
\documentclass{article} %

\usepackage{iclr2027_conference, times}
\usepackage{iclr_preprint_compat}
\iclrpreprintcopy

\usepackage{microtype}
\usepackage{hyperref}
\usepackage{url}
\usepackage{booktabs}

\usepackage{graphicx}
\usepackage{amsmath,amssymb,amsthm,mathtools}
\usepackage{enumitem}
\usepackage{float}
\usepackage{algorithm}
\usepackage{algpseudocode}
\usepackage{tikz}
\usetikzlibrary{arrows.meta}
\usepackage[capitalize,noabbrev]{cleveref}
\usepackage{thm-restate}
\usepackage{xcolor} %
\usepackage{multirow}
\usepackage{subcaption}
\usepackage{wrapfig}

\usepackage[most]{tcolorbox}
\newtcolorbox{samplebox}[1]{
    colback=blue!3!white,
    colframe=black!50!white,
    colbacktitle=black!50!white,
    coltitle=white,
    boxrule=0.5pt,
    arc=1mm,
    left=2pt, right=2pt, top=2pt, bottom=2pt,
    boxsep=1pt,
    toptitle=1pt, bottomtitle=1pt,
    title={#1}
}

\usepackage[most]{tcolorbox}
\definecolor{keylavender}{RGB}{245,245,250}
\tcbset{bofflab-common/.style={
  enhanced, breakable,
  boxsep=6pt, left=8pt, right=8pt, top=6pt, bottom=6pt,
  before skip=8pt, after skip=8pt,
  arc=10pt
}}
\newtcolorbox{lavenderbox}{
  bofflab-common,
  colback=keylavender,
  colframe=keylavender,
  boxrule=0pt,
  arc=5pt,
  left=3pt, right=3pt, top=0pt, bottom=0pt
}

\theoremstyle{plain}
\newtheorem{theorem}{Theorem}[section]
\newtheorem{proposition}[theorem]{Proposition}
\newtheorem{lemma}[theorem]{Lemma}

\theoremstyle{definition}
\newtheorem{definition}[theorem]{Definition}

\crefformat{section}{Section~#2#1#3}
\crefformat{subsection}{Section~#2#1#3}
\crefformat{subsubsection}{Section~#2#1#3}
\crefformat{equation}{(#2#1#3)}
\crefrangeformat{equation}{(#3#1#4) to (#5#2#6)}
\crefmultiformat{equation}{(#2#1#3)}{ and (#2#1#3)}{, (#2#1#3)}{ and (#2#1#3)}
\crefname{appendix}{Appendix}{Appendices}
\Crefname{appendix}{Appendix}{Appendices}
\crefformat{appendix}{Appendix~#2#1#3}
\let\SCFMLMorigappendix\appendix
\renewcommand{\appendix}{%
\SCFMLMorigappendix
\crefalias{section}{appendix}%
\crefalias{subsection}{appendix}%
\crefalias{subsubsection}{appendix}%
}
\crefname{table}{Table}{Tables}
\crefname{figure}{Figure}{Figures}
\crefname{lemma}{Lemma}{Lemmas}
\crefname{proposition}{Proposition}{Propositions}
\crefname{assumption}{Assumption}{Assumptions}

\graphicspath{{figures/markdown_export/}}

\usepackage{lineno}

\definecolor{darkblue}{rgb}{0, 0, 0.5}
\hypersetup{colorlinks=true, citecolor=darkblue, linkcolor=darkblue, urlcolor=darkblue, hypertexnames=false}

\newcommand{\draft}[1]{{#1}}
\definecolor{claimpink}{RGB}{210,45,125}

\definecolor{nbpurple}{rgb}{0.5, 0, 0.5}

\newcommand{\cutabstractup}{\vspace*{-0.2in}}

\DeclareMathOperator*{\argmin}{argmin}

\title{Self-conditioned Flow Map Language Models via Fixed-point Flows}

\author{%
    \parbox{0.8\linewidth}{
    Jaehoon Yoo$^1$\thanks{Equal contribution $^\dag$ Equal advising}\hspace{1em}
    Wonjung Kim$^1$\footnotemark[1]\hspace{1em}
    Floor Eijkelboom$^2$\hspace{1em}
    Chanhyuk Lee$^1$\\
    Nicholas M. Boffi$^3$\hspace{1em}
    Seunghoon Hong$^1$$^\dag$\hspace{1em}
    Jinwoo Kim$^1$$^\dag$
    }\vspace{0.1cm}\\
    $^1$KAIST\hspace{1em}
    $^2$University of Amsterdam\hspace{1em}
    $^3$Carnegie Mellon University
}

\begin{document}

\maketitle

\cutabstractup
\begin{abstract}
Self-conditioning is a core technique that enhances continuous flow-based language models, where the model learns to denoise generated text by conditioning on its own denoising estimate.
While empirically successful, its performance improvements are poorly understood.
Moreover, there is growing interest in the use of few-step generators based on flow maps, for which how to leverage self-conditioning is unclear.
{Here, we show that flow language models with self-conditioning perform a fixed-point iteration that improves generation through iterative refinement.}
We use this viewpoint to formulate \textbf{fixed-point flows}, a two-dimensional class of self-conditioned flows, where the first dimension represents the flow process and the second represents the fixed-point iteration.
We show that fixed-point flows define valid flow maps, and show that they can be distilled from self-conditioned flow models by compressing both fixed-point iterations and the flow process, the former with fixed-point distillation and the latter with flow map distillation.
Our resulting flow map language model, FMLM$^\star$, outperforms state-of-the-art self-conditioned models and few-step models in one- and few-step generation on OpenWebText.\footnote{Code is available at \url{https://github.com/Ugness/self-conditioned-fmlm}.}
\end{abstract}

\input{sections/introduction}

\input{sections/background}

\input{sections/fixed_point_refinement}

\input{sections/distillation}

\input{sections/experiments}

\input{sections/related_work}

\input{sections/conclusion}

\bibliography{references}
\bibliographystyle{iclr2027_conference}

\newpage
\appendix

\input{sections/appendix_proofs}

\input{sections/appendix_implementation}

\input{sections/appendix_results}

\input{sections/appendix_samples}

\end{document}

%% file: sections/introduction.tex
\section{Introduction}
\label{sec:intro}

Language models (LMs) based on continuous flows have recently emerged as a promising paradigm for non-autoregressive text generation~\citep{lee2026flow,chemseddine2026spherical,deschenaux2026language}.
By learning to denoise tokens in a continuous space, these models perform parallel iterative generation through a deterministic evolution driven by a velocity field.
Importantly, such processes define a unique flow map, the solution operator that directly transports noise to data in as few as one function evaluation.
This advantage has sparked recent breakthroughs on distillation of flow language models into flow map language models that are capable of generating text in one to few inference steps~\citep{lee2026flow,roos2026categorical,potaptchik2026discrete}.

Recently, the performance of flow language models has improved through the incorporation of a self-conditioning mechanism~\citep{chen2022analog}.
Unlike in usual flow training, self-conditioned flow models learn to denoise an input by conditioning on its own denoising prediction.
Then, during generation, they perform denoising at each flow timestep by conditioning on the previously denoised outcome.
Self-conditioning has empirically proven very effective~\citep{dieleman2022continuous,strudel2022self}, and has been widely adopted in the latest state-of-the-art flow language models~\citep{chen2026langflow,hu2026elf,batzolis2026towards,meshchaninov2026train,yang2026continuous}.
Yet, despite its success, why self-conditioning leads to improvements remains poorly understood.
Furthermore, it is unclear how to distill flows into flow maps under self-conditioning, as it introduces additional dependencies across generation timesteps.

In this work, we introduce \textbf{fixed-point flows}, a mathematical framework for self-conditioned flows and their associated flow maps (\Cref{fig:overview}).
{Our key insight is that self-conditioning in flow language models can be viewed as a fixed-point iteration.}
{With this insight, we characterize a two-dimensional class of self-conditioned flows, where the first dimension represents the original flow, and the second represents fixed-point iterations, which empirically improve generation quality.}
We use this view to distill self-conditioned flow language models into flow maps by compressing both fixed-point iterations and flow, achieving state-of-the-art one- and few-step language modeling.

Our main contributions are:
\begin{itemize}[leftmargin=*]
  \item \textbf{Fixed-point view of self-conditioning.} {We show that flow language models with self-conditioning implicitly learn a fixed-point iteration.} {We theoretically establish convergence of this iteration under a contractivity assumption.}
  \item \textbf{Fixed-point flows and their flow maps.} We introduce fixed-point flows, a formalization of self-conditioned flows. We show that fixed-point flows define a valid flow map, which can be learned by compressing both fixed-point iterations and flow with the respective few-step distillation objectives.
  \item \textbf{Empirical results.} {We empirically show that refining predictions toward the learned fixed point improves generation quality.} We distill self-conditioned flow language models into flow map language models, achieving state-of-the-art one- and few-step generation on OpenWebText. We also show that it is possible to distill self-conditioned models into self-conditioning-free ones without degradation.
\end{itemize}

%% file: sections/background.tex
\section{Preliminary}
\label{sec:background}
\begin{figure}[t!]
  \centering
  \input{figures/diagrams2}
  \caption{\textbf{Overview.}
  {We show that flow language models with self-conditioning perform a fixed-point iteration that improves generation quality.}
  We leverage this insight to formulate fixed-point flows, a class of self-conditioned flows that run fixed-point iterations at each flow timestep.
  Fixed-point flows yield valid flow maps, which we learn by compressing both the flow and the fixed-point iterations.
  }
  \label{fig:overview}
  \vspace{-6pt} %
\end{figure}
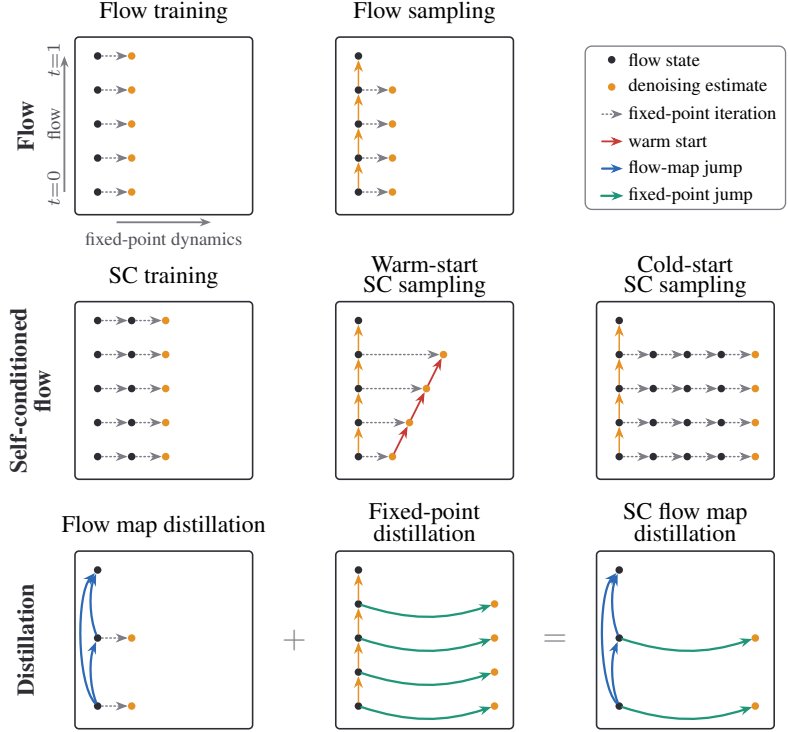

\paragraph{Flow language models.}

Let $V$ be a vocabulary of tokens, and denote text data with length $L$ by ${\bf y} = ({\bf y}^l)_{l=1}^L \in V^L$.
The goal of language modeling is to learn the data distribution $p({\bf y})$ so that we can draw a new text sample efficiently.
Flow language models choose a continuous embedding ${\bf y}\mapsto {\bf x}\in\mathbb{R}^{L\times d}$ and a decoder ${\bf x}\mapsto {\bf y}$, and model the induced distribution $p({\bf x})$ on the embedding space, generating $\hat{{\bf x}}\sim p({\bf x})$ and then rounding into discrete language through $\hat{{\bf x}}\mapsto \hat{{\bf y}}$.

To model the now-continuous data distribution $p({\bf x})$, flow language models use flow matching over a stochastic interpolant~\citep{lipman2022flow,albergo2025stochastic}.
They specify a probability path $p_t({\bf x}_t)$ over $t\in[0,1]$ as the density of an interpolant $I_t \coloneqq (1-t){\bf x}_0 + t{\bf x}_1$ between noise ${\bf x}_0\sim p_0$ and data ${\bf x}_1\sim p_1$.
The path $p_t$ admits a deterministic evolution equation that can be used to draw a sample ${\bf x}_t\sim p_t$ at inference time:
\begin{equation}\label{eq:flow_velocity}
    \dot{{\bf x}}_t = b_t({\bf x}_t),\quad b_t({\bf x}) = \mathbb{E}[{\bf x}_1-{\bf x}_0 \mid I_t = {\bf x}],\quad {\bf x}_0\sim p_0,\quad t\in[0,1],
\end{equation}
where $b_t$ is the velocity field of the flow.
If a model $\hat{b}$ of the velocity is available, a sample $\hat{{\bf x}}_1\sim p_1$ can be approximately drawn by numerically integrating \eqref{eq:flow_velocity} across a time grid $0=t_0<...<t_N=1$.
A simple choice for the integration is the forward Euler scheme:
\begin{equation}\label{eq:euler}
    \hat{{\bf x}}_{t_{i+1}} = \hat{{\bf x}}_{t_i} + (t_{i+1}-t_i)\hat{b}_{t_i}(\hat{{\bf x}}_{t_i}),\quad \hat{{\bf x}}_0\sim p_0.
\end{equation}
Instead of learning the velocity directly, it is more common in language modeling to learn the denoiser $D_t$~\citep{lee2026flow,hu2026elf} which outputs the mean of clean data, and from which one can recover the velocity~\citep{albergo2025stochastic,li2026back}:
\begin{equation}\label{eq:denoiser}
    D_t({\bf x}) \coloneqq \mathbb{E}[{\bf x}_1\mid I_t={\bf x}],\quad b_t({\bf x}) = \frac{D_t({\bf x}) - {\bf x}}{1-t}.
\end{equation}
The ideal denoiser $D$ can be learned by solving a regression problem $D = \argmin_{\hat{D}}\mathcal{L}(\hat{D})$ that predicts clean data by minimizing the following loss:
\begin{equation}\label{eq:denoiser_loss}
    \mathcal{L}(\hat{D}) \coloneqq \int_0^1 \mathbb{E}|\hat{D}_t(I_t) - {\bf x}_1|^2{\rm d}t.
\end{equation}
While we have used square loss above, our discussions and results readily transfer to the cross-entropy setting of \citet{dieleman2022continuous,eijkelboom2024variational, lee2026flow}.
With a learned denoiser $\hat{D}$, generation can be done by turning it into an estimation of the velocity $\hat{b}$ through \eqref{eq:denoiser} and then leveraging the forward Euler scheme \eqref{eq:euler}.

\paragraph{Flow map language models.}
A key advantage of continuous flow for language modeling is that its deterministic generative process~\eqref{eq:flow_velocity} driven by the velocity $b_t$ defines a unique flow map $X_{s,t}:\mathbb{R}^{L\times d}\to \mathbb{R}^{L\times d}$, the solution operator that transports a sample directly ${\bf x}_t=X_{s,t}({\bf x}_s)$ between any timesteps $(s,t)\in[0,1]^2$.
It satisfies the following integral equation:
\begin{equation}\label{eq:flow_map}
    X_{s,t}({\bf x}_s) = {\bf x}_s + \int_s^t b_\tau({\bf x}_\tau)\,{\rm d}\tau.
\end{equation}
If available, a flow map allows for efficient few-step generation by sequentially evaluating $\hat{{\bf x}}_{t_{i+1}} = X_{t_i,t_{i+1}}(\hat{{\bf x}}_{t_i})$ upon any grid $0=t_0<...<t_N=1$, even one-step generation via $\hat{{\bf x}}_1 = X_{0,1}(\hat{{\bf x}}_0)$.

Recent works have demonstrated that flow map language models can be learned by distilling the flow velocity learned by flow language models~\citep{lee2026flow,roos2026categorical,potaptchik2026discrete}.
These methods leverage a set of mathematical relations between the flow velocity and the flow map~\citep{boffi2026build, lee2026flow}, which are derivable from the integral equation in \eqref{eq:flow_map}, and learn the flow map with respect to a given flow velocity to satisfy these relations.

%% file: figures/diagrams2.tex
\definecolor{cflow}{RGB}{45,105,180}
\definecolor{cfp}{RGB}{32,150,120}
\definecolor{ctraj}{RGB}{230,145,35}
\definecolor{cwarm}{RGB}{200,60,55}
\definecolor{cstate}{RGB}{45,45,52}
\definecolor{caxis}{RGB}{125,125,133}
\definecolor{caxistext}{RGB}{80,80,88}

\begin{tikzpicture}[
  box/.style={draw=cstate, line width=.6pt, rounded corners=1.5pt},
  pt/.style={circle, fill=cstate, inner sep=1.0pt},
  fp/.style={circle, fill=ctraj, inner sep=1.0pt},
  arr/.style={-{Stealth[length=1.6mm,width=1.2mm]}, line width=.55pt, line cap=round},
  arrdot/.style={arr, shorten <=1.5pt, shorten >=1.5pt},
  fpidot/.style={arrdot, caxis, densely dotted},
  trajarr/.style={arrdot, ctraj},
  warmarr/.style={arrdot, cwarm, line width=.7pt},
  flowarr/.style={arrdot, cflow, line width=.85pt},
  fparr/.style={arr, cfp, line width=.85pt, shorten <=1.2pt, shorten >=1.2pt},
  axisarr/.style={-{Stealth[length=1.5mm,width=1.1mm]}, line width=.7pt, caxis},
  title/.style={font=\small, align=center},
  rowlab/.style={font=\footnotesize\bfseries, align=center, rotate=90, text=cstate},
  lab/.style={font=\scriptsize, text=caxistext},
  tlab/.style={font=\scriptsize, text=caxistext}
]

\def\S{2.35}
\def\xA{0.30}\def\xB{0.75}\def\xC{1.20}\def\xD{1.65}\def\xE{2.10}
\def\xb{0.975}\def\xc{1.425}
\def\yA{0.30}\def\yB{0.75}\def\yC{1.20}\def\yD{1.65}\def\yE{2.10}
\def\rT{6.8}\def\rM{3.3}\def\rB{0.0}
\def\cA{0.0}\def\cB{3.45}\def\cC{6.90}

\node[rowlab] at (-0.65,\rT+\S/2) {Flow};
\node[rowlab] at (-0.59,\rM+\S/2) {Self-conditioned\\[-1pt]flow};
\node[rowlab] at (-0.65,\rB+\S/2) {Distillation};

\begin{scope}[shift={(\cA,\rT)}]
  \node[title] at (\S/2,\S+0.32) {Flow training};
  \draw[box] (0,0) rectangle (\S,\S);
  \foreach \y in {\yE,\yD,\yC,\yB,\yA}{
    \node[pt] at (\xA,\y){}; \node[fp] at (\xB,\y){};
    \draw[fpidot] (\xA,\y)--(\xB,\y);}
  \draw[axisarr] (-0.14,0.30) -- (-0.14,2.10);
  \node[lab, rotate=90, anchor=center] at (-0.28,1.20) {flow};
  \node[tlab, rotate=90, anchor=center] at (-0.28,0.35) {$t{=}0$};
  \node[tlab, rotate=90, anchor=center] at (-0.28,2.05) {$t{=}1$};
  \draw[axisarr] (0.55,-0.10) -- (1.80,-0.10);
  \node[lab, anchor=north] at (1.175,-0.11) {fixed-point dynamics};
\end{scope}

\begin{scope}[shift={(\cB,\rT)}]
  \node[title] at (\S/2,\S+0.32) {Flow sampling};
  \draw[box] (0,0) rectangle (\S,\S);
  \foreach \y in {\yD,\yC,\yB,\yA}{
    \node[pt] at (\xA,\y){}; \node[fp] at (\xB,\y){};
    \draw[fpidot] (\xA,\y)--(\xB,\y);}
  \node[pt] at (\xA,\yE){};
  \foreach \ys/\ye in {\yA/\yB,\yB/\yC,\yC/\yD,\yD/\yE}{\draw[trajarr] (\xA,\ys)--(\xA,\ye);}
\end{scope}

\begin{scope}[shift={(\cC-0.1,\rT+0.05)}]
  \def\dx{0.30}\def\ax{0.12}\def\axe{0.50}\def\tx{0.40}
  \node[pt] at (\dx,2.00){};
  \node[font=\scriptsize,anchor=west] (L1) at (\tx,2.00) {flow state};
  \node[fp] at (\dx,1.64){};
  \node[font=\scriptsize,anchor=west] (L2) at (\tx,1.64) {denoising estimate};
  \draw[fpidot] (\ax,1.28)--(\axe,1.28);
  \node[font=\scriptsize,anchor=west] (L3) at (\tx,1.28) {fixed-point iteration};
  \draw[warmarr] (\ax,0.92)--(\axe,0.92);
  \node[font=\scriptsize,anchor=west] (L4) at (\tx,0.92) {warm start};
  \draw[flowarr] (\ax,0.56)--(\axe,0.56);
  \node[font=\scriptsize,anchor=west] (L5) at (\tx,0.56) {flow-map jump};
  \draw[fparr] (\ax,0.20)--(\axe,0.20);
  \node[font=\scriptsize,anchor=west] (L6) at (\tx,0.20) {fixed-point jump};
  \draw[draw=caxis, line width=.5pt, rounded corners=2.5pt] (-0.05,0.02) rectangle (2.6,2.18);
\end{scope}

\begin{scope}[shift={(\cA,\rM)}]
  \node[title] at (\S/2,\S+0.32) {SC training};
  \draw[box] (0,0) rectangle (\S,\S);
  \foreach \y in {\yE,\yD,\yC,\yB,\yA}{
    \node[pt] at (\xA,\y){}; \node[pt] at (\xB,\y){}; \node[fp] at (\xC,\y){};
    \draw[fpidot] (\xA,\y)--(\xB,\y); \draw[fpidot] (\xB,\y)--(\xC,\y);}
\end{scope}

\begin{scope}[shift={(\cB,\rM)}]
  \node[title] at (\S/2,\S+0.50) {Warm-start};
  \node[title] at (\S/2,\S+0.2) {SC sampling};
  \draw[box] (0,0) rectangle (\S,\S);
  \node[pt] at (\xA,\yA){}; \node[fp] at (\xB,\yA){}; \draw[fpidot] (\xA,\yA)--(\xB,\yA);
  \node[pt] at (\xA,\yB){}; \node[fp] at (\xb,\yB){}; \draw[fpidot] (\xA,\yB)--(\xb,\yB);
  \node[pt] at (\xA,\yC){}; \node[fp] at (\xC,\yC){}; \draw[fpidot] (\xA,\yC)--(\xC,\yC);
  \node[pt] at (\xA,\yD){}; \node[fp] at (\xc,\yD){}; \draw[fpidot] (\xA,\yD)--(\xc,\yD);
  \node[pt] at (\xA,\yE){};
  \draw[warmarr] (\xB,\yA)--(\xb,\yB);
  \draw[warmarr] (\xb,\yB)--(\xC,\yC);
  \draw[warmarr] (\xC,\yC)--(\xc,\yD);
  \foreach \ys/\ye in {\yA/\yB,\yB/\yC,\yC/\yD,\yD/\yE}{\draw[trajarr] (\xA,\ys)--(\xA,\ye);}
\end{scope}

\begin{scope}[shift={(\cC,\rM)}]
  \node[title] at (\S/2,\S+0.50) {Cold-start};
  \node[title] at (\S/2,\S+0.2) {SC sampling};
  \draw[box] (0,0) rectangle (\S,\S);
  \foreach \y in {\yD,\yC,\yB,\yA}{
    \node[pt] at (\xA,\y){}; \node[pt] at (\xB,\y){}; \node[pt] at (\xC,\y){};
    \node[pt] at (\xD,\y){}; \node[fp] at (\xE,\y){};
    \draw[fpidot] (\xA,\y)--(\xB,\y); \draw[fpidot] (\xB,\y)--(\xC,\y);
    \draw[fpidot] (\xC,\y)--(\xD,\y); \draw[fpidot] (\xD,\y)--(\xE,\y);}
  \node[pt] at (\xA,\yE){};
  \foreach \ys/\ye in {\yA/\yB,\yB/\yC,\yC/\yD,\yD/\yE}{\draw[trajarr] (\xA,\ys)--(\xA,\ye);}
\end{scope}

\begin{scope}[shift={(\cA,\rB)}]
  \node[title] at (\S/2,\S+0.32) {Flow map distillation};
  \draw[box] (0,0) rectangle (\S,\S);
  \node[pt] at (\xA,\yA){}; \node[fp] at (\xB,\yA){}; \draw[fpidot] (\xA,\yA)--(\xB,\yA);
  \node[pt] at (\xA,\yC){}; \node[fp] at (\xB,\yC){}; \draw[fpidot] (\xA,\yC)--(\xB,\yC);
  \node[pt] at (\xA,\yE){};
  \draw[flowarr] (\xA,\yA) to[out=112,in=-112,looseness=.85] (\xA,\yC);
  \draw[flowarr] (\xA,\yC) to[out=112,in=-112,looseness=.85] (\xA,\yE);
  \draw[flowarr] (\xA,\yA) to[out=132,in=-132,looseness=.60] (\xA,\yE);
\end{scope}

\node[font=\large\bfseries, text=caxis] at (\cA+\S+0.55,\rB+\S/2) {$+$};

\begin{scope}[shift={(\cB,\rB)}]
  \node[title] at (\S/2,\S+0.50) {Fixed-point};
  \node[title] at (\S/2,\S+0.24) {distillation};
  \draw[box] (0,0) rectangle (\S,\S);
  \foreach \y in {\yA,\yB,\yC,\yD,\yE}{\node[pt] at (\xA,\y){};}
  \foreach \y in {\yA,\yB,\yC,\yD}{
    \draw[fparr] (\xA,\y) to[bend right=18] (\xE,\y);
    \node[fp] at (\xE,\y){};}
  \foreach \ys/\ye in {\yA/\yB,\yB/\yC,\yC/\yD,\yD/\yE}{\draw[trajarr] (\xA,\ys)--(\xA,\ye);}
\end{scope}

\node[font=\large\bfseries, text=caxis] at (\cB+\S+0.55,\rB+\S/2) {$=$};

\begin{scope}[shift={(\cC,\rB)}]
  \node[title] at (\S/2,\S+0.50) {SC flow map};
  \node[title] at (\S/2,\S+0.24) {distillation};
  \draw[box] (0,0) rectangle (\S,\S);
  \node[pt] at (\xA,\yA){}; \node[pt] at (\xA,\yC){}; \node[pt] at (\xA,\yE){};
  \draw[fparr] (\xA,\yA) to[bend right=18] (\xE,\yA);
  \draw[fparr] (\xA,\yC) to[bend right=18] (\xE,\yC);
  \node[fp] at (\xE,\yA){}; \node[fp] at (\xE,\yC){};
  \draw[flowarr] (\xA,\yA) to[out=112,in=-112,looseness=.85] (\xA,\yC);
  \draw[flowarr] (\xA,\yC) to[out=112,in=-112,looseness=.85] (\xA,\yE);
  \draw[flowarr] (\xA,\yA) to[out=132,in=-132,looseness=.60] (\xA,\yE);
\end{scope}

\end{tikzpicture}

%% file: sections/fixed_point_refinement.tex
\section{Theoretical framework}
\label{sec:theory}

{In this section, we formalize self-conditioning and propose that it learns a fixed-point iteration which refines the model's prediction.}
We then develop fixed-point flows, a framework for self-conditioned flows, along with the associated flow maps for few-step generation.
Finally, we introduce distillation methods for turning self-conditioned flow language models into few-step language models.

\vspace{-4pt}
\subsection{Self-conditioned flow language models}\label{sec:sc_flm}

Self-conditioning is a technique introduced in \citet{chen2022analog} where the denoiser learns to predict conditioned on its own denoising estimates.
In these approaches, the model of the denoiser \eqref{eq:denoiser} takes an additional conditioning, $\hat{D}_t({\bf x}, {\bf z})$, which effectively reduces to a usual denoiser model when ${\bf z}=\boldsymbol{0}$.
The denoiser is trained using the following loss, where $\mu$ is a choice of training-time distribution of ${\bf z}$:
\begin{equation}\label{eq:sc_loss}
    \mathcal{L}_\mu(\hat{D}) \coloneqq \int_0^1\mathbb{E}_{{\bf x}_0,{\bf x}_1}\mathbb{E}_\mu|\hat{D}_t(I_t, {\bf z}) - {\bf x}_1|^2{\rm d}t.
\end{equation}
A common choice of $\mu$ is the mixture of delta peaks at zero ${\bf z}=\boldsymbol{0}$ and the model's own denoising prediction ${\bf z} = \mathsf{sg}(\hat{D}_t(I_t, \boldsymbol{0}))$ where $\mathsf{sg}$ is the stop-gradient operator.
This leads to the following:
\begin{equation}\label{eq:sc_loss_decomposed}
    \mathcal{L}_\mu(\hat{D}) = \frac{1}{2}\int_0^1\mathbb{E}|\hat{D}_t(I_t, \boldsymbol{0}) - {\bf x}_1|^2{\rm d}t + \frac{1}{2}\int_0^1\mathbb{E}|\hat{D}_t(I_t, \mathsf{sg}(\hat{D}_t(I_t, \boldsymbol{0})))- {\bf x}_1|^2 {\rm d}t,
\end{equation}
where the first term is the usual denoising loss \eqref{eq:denoiser_loss}, whereas the second term is a self-conditioned loss which encourages correcting an initial estimate by the model.

In general, the conditioning variable ${\bf z}$ in the training loss \eqref{eq:sc_loss} does not change the learning target of the denoiser.
This is because the Bayes-optimal prediction of ${\bf x}_1$ given an interpolant $I_t$ is specified by the ideal denoiser, $D_t(I_t)$~\eqref{eq:denoiser}, regardless of ${\bf z}$, assuming ${\bf z}$ does not leak additional information about ${\bf x}_1$ beyond what is available from $I_t$, as in \eqref{eq:sc_loss_decomposed}.
Thus, the objective trains the model to map every on-distribution conditioning ${\bf z}$ to the ideal denoising target.
We provide a formal proof:
\begin{lavenderbox}
    \begin{restatable}{proposition}{optimality}\label{prop:optimality}
        Assume that ${\bf z}\perp {\bf x}_1\mid I_t$ for almost every $t$, and assume all second moments are finite.
        Then every unrestricted population minimizer $\bar{D}$ of \eqref{eq:sc_loss} satisfies
        \begin{equation}
            \bar{D}_t({\bf x},{\bf z}) = D_t({\bf x})
        \end{equation}
        almost surely under the law of $(I_t, {\bf z})$, for almost every $t$.
\end{restatable}
\end{lavenderbox}
A proof is in \Cref{app:proof_optimality}.
In practice, training would not attain the minimum loss and the model would have dependence on ${\bf z}$.
With the self-conditioned loss \eqref{eq:sc_loss_decomposed}, this can encourage a self-correcting behavior that produces an initial prediction and then pulls it closer to the Bayes-optimal prediction.

The generative process under self-conditioning is similar to that of \eqref{eq:flow_velocity}, evolving a sample $\hat{{\bf x}}_{t_i}$ across a flow-time grid $0=t_0<...<t_N=1$.
Yet, the denoising estimate is also maintained and updated as $\hat{{\bf z}}_{t_i}$ to condition and bootstrap the denoiser, altering the numerical integration scheme \eqref{eq:euler} as follows:
\begin{equation}\label{eq:sc_euler}
    \begin{aligned}
        \hat{{\bf x}}_{t_{i+1}} &= \hat{{\bf x}}_{t_i} + (t_{i+1} - t_i) \hat{b}_{t_i},\quad \hat{b}_{t_i} = \frac{\hat{{\bf z}}_{t_{i+1}} - \hat{{\bf x}}_{t_i}}{1-t_i},\quad \hat{{\bf x}}_0\sim p_0,\\
        \hat{{\bf z}}_{t_{i+1}} &= \hat{D}_{t_i}(\hat{{\bf x}}_{t_i}, \hat{{\bf z}}_{t_i}),\quad \hat{{\bf z}}_0 = \boldsymbol{0}.
    \end{aligned}
\end{equation}
Intuitively, this creates an information sharing across flow timesteps via $\hat{\bf z}$ on top of the flow state $\hat{\bf x}$.

Although self-conditioning empirically leads to substantial improvements of flow language models~\citep{chen2026langflow,hu2026elf,batzolis2026towards,meshchaninov2026train,yang2026continuous}, understanding of its mechanism is limited~\citep{chen2022analog,dieleman2022continuous,shabalin2025tencdm,shabalin2026gaussian}.
Furthermore, since the self-conditioned velocity \eqref{eq:sc_euler} is also coupled with the updates of the denoising estimates, it is unclear how to use it for few-step distillation, which relies on a flow map defined by integrating the velocity field as in \eqref{eq:flow_map}.
\vspace{-4pt}
\subsection{Self-conditioning induces a fixed-point iteration}
\label{sec:formalization}
{We present our key intuition that a self-conditioned denoiser, trained with \eqref{eq:sc_loss_decomposed}, refines its own predictions through a fixed-point iteration.}
Formally, a self-conditioned denoiser $\hat{D}({\bf x},{\bf z})$, given a choice of flow timestep $t$ and flow state ${\bf x}$, can be used to define the following iteration ${\bf z}^j$ for $j = 0, 1,...$ that starts at some initialization ${\bf z}^0$:
\begin{equation}\label{eq:sc_map}
    {\bf z}^{j+1} = \hat{D}_t({\bf x}, {\bf z}^j).
\end{equation}
As the denoiser learns to correct its own prediction \eqref{eq:sc_loss_decomposed}, we posit that it characterizes a self-correcting process, eventually reaching a fixed point ${\bf z}^\star = \hat{D}_t({\bf x}, {\bf z}^\star)$~\citep{gulrajani2023likelihood}.
{We formalize this intuition below and show that the fixed point coincides with the Bayes-optimal prediction in the ideal case.}

First, we recall the notion of a contraction, which suffices for convergence to a unique fixed point.
\begin{lavenderbox}
    \begin{restatable}[\bf Contraction]{definition}{contraction}\label{def:contraction}
        A function $f:\mathbb{R}^d\to \mathbb{R}^d$ is a contraction on a closed set $O\subseteq \mathbb{R}^d$ with factor $0\leq \eta<1$ if it satisfies $f(O)\subseteq O$ and $|f({\bf z}) - f({\bf z}')| \leq \eta |{\bf z}-{\bf z}'|$ for every ${\bf z},{\bf z}'\in O$.
    \end{restatable}
\end{lavenderbox}
By \Cref{prop:optimality}, a perfect denoiser $\bar{D}_t({\bf x},\cdot)$ with minimum loss maps directly to $D_t({\bf x})$, making it a contraction with $\eta=0$.
When the denoiser $\hat{D}$ is learned, it is not always guaranteed to be contractive.
Yet, in \Cref{app:proof_contractivity}, we show that optimizing self-conditioned loss leads to approximate contractivity.
To simplify our analysis, we shall henceforth assume that the learned denoiser is contractive.
This is in line with common assumptions in looped models~\citep{romano2017little,ryu2019plug,bai2019deep,fung2022jfb,movahedi2026fixed} and supported by our results that self-conditioned models ELF~\citep{hu2026elf} and LangFlow~\citep{chen2026langflow} approximate fixed-point iterations.
We show:
\begin{lavenderbox}
    \begin{restatable}{proposition}{learnedfixedpoint}\label{prop:learnedfixedpoint}
        Fix $t$, ${\bf x}$, and $0\leq\eta < 1$.
        Assume that $\hat{D}_t({\bf x}, \cdot)$ is a contraction on a nonempty closed set $O\subseteq\mathbb{R}^{L\times d}$ with factor $\eta$.
        Then, the iteration ${\bf z}^{j+1} = \hat{D}_t({\bf x}, {\bf z}^j)$ \eqref{eq:sc_map} from any ${\bf z}^0\in O$ satisfies the following properties:
        \begin{enumerate}[label=(\roman*),leftmargin=*]
            \item It converges exponentially to a unique fixed point ${\bf z}^\star\in O$,
            \item Its denoising error at any iteration $j\geq 1$ is bounded as
            \begin{equation}
                |{\bf z}^j - D_t({\bf x})| \leq |\hat{D}_t({\bf x},{\bf z}^0) - D_t({\bf x})| + \frac{\eta-\eta^j}{1-\eta} |{\bf z}^1 - {\bf z}^0|.
            \end{equation}
            \item Its denoising error at the fixed point is bounded as
            \begin{equation}
                |{\bf z}^\star - D_t({\bf x})| \leq |\hat{D}_t({\bf x},{\bf z}^0) - D_t({\bf x})| + \frac{\eta}{1-\eta} |{\bf z}^1 - {\bf z}^0|,
            \end{equation}
        \end{enumerate}
    \end{restatable}
\end{lavenderbox}
A proof is in \Cref{app:proof_learnedfixedpoint}.
The result shows that the initial prediction error, together with the magnitude of the first iteration, controls the denoising error bound up to the fixed point. It further guarantees convergence to this fixed point under contractivity.

Having formulated the relationship between a self-conditioned denoiser $\hat{D}$ and its fixed points for each pair of $t$ and ${\bf x}$, it is convenient to think of a function that directly predicts the fixed point.
We call such a function the \textbf{fixed-point denoiser} $D^\star$, defined as follows for every $t\in[0,1],{\bf x}\in\mathbb{R}^{L\times d}$:
\begin{equation}\label{eq:fp_denoiser}
    D_t^\star({\bf x}) \coloneqq {\bf z}^\star = \hat{D}_t({\bf x}, {\bf z}^\star).
\end{equation}
where ${\bf z}^\star$ is the unique fixed point of $\hat{D}_t({\bf x}, \cdot)$.
The fixed-point denoiser is a self-conditioning-free denoiser that behaves the same as the given self-conditioned denoiser run until convergence.
Its ideal target under \Cref{prop:optimality} is the Bayes-optimal prediction $D_t({\bf x})$, in which case we have $D^\star=D$.

\vspace{-4pt}
\subsection{Fixed-point flows}
\label{sec:fpflow}
In \Cref{sec:sc_flm}, we have seen that self-conditioned flow appears to involve two states: the flow state and the self-conditioning state.
In \Cref{sec:formalization}, we have established that a self-conditioned denoiser induces a fixed-point iteration.
Here, we propose to replace the self-conditioning state by its fixed point.
Once this is done, the velocity depends only on the flow state, so it becomes an ordinary flow.
We call this a \textbf{fixed-point flow}, and prove its fundamental properties as follows.
\begin{lavenderbox}
    \begin{restatable}{proposition}{fpflow}
    \label{prop:fpflow}
    Assume that, for every $t\in[0,1)$ and ${\bf x}$, the self-conditioned denoiser $\hat{D}_t({\bf x},\cdot)$ has a unique fixed point $D^\star_t({\bf x}) = \hat{D}_t({\bf x},D^\star_t({\bf x}))$.
    Define the \textbf{fixed-point velocity}
    \begin{equation}\label{eq:fp_velocity}
        b_t^\star({\bf x}) \coloneqq \frac{D_t^\star({\bf x}) - {\bf x}}{1-t},\quad t\in[0,1).
    \end{equation}
    Then $b_t^\star$ depends only on $t$ and ${\bf x}$, and therefore the dynamics
    \begin{equation}\label{eq:fp_ode}
        \dot{{\bf x}}_t = b_t^\star({\bf x}_t)
    \end{equation}
    define an ordinary time-dependent ODE.
    Furthermore, when the fixed-point denoiser is Bayes optimal, $D^\star=D$, the fixed-point velocity recovers the true velocity, $b^\star = b$.
    \end{restatable}
\end{lavenderbox}
A proof is in \Cref{app:proof_fpflow}.
Equation~\eqref{eq:fp_velocity} says that at each flow time, the self-conditioned denoiser $\hat{D}$ runs fixed-point iteration \eqref{eq:sc_map} until convergence, and the fixed point modulates the flow via \eqref{eq:fp_velocity}.
This suggests an Euler scheme for sampling $\hat{{\bf x}}_1\sim p_1$ over a flow-time grid $0=t_0<...<t_N=1$:
\begin{equation}
    \hat{{\bf x}}_{t_{i+1}} = \hat{{\bf x}}_{t_i} + (t_{i+1} - t_i) \frac{\hat{{\bf z}}_{t_i}^\star - \hat{{\bf x}}_{t_i}}{1-t_i},
\end{equation}
where each $\hat{{\bf z}}_{t_i}^\star$ is obtained by solving an inner fixed-point iteration from an initialization $\hat{{\bf z}}_{t_i}^0$:
\begin{equation}\label{eq:inner_fpi}
    \hat{{\bf z}}_{t_i}^{j+1} = \hat{D}_{t_i}(\hat{{\bf x}}_{t_i}, \hat{{\bf z}}_{t_i}^j).
\end{equation}
{Under the contractivity assumption, the fixed-point iteration converges to a unique fixed point regardless of its initialization within the contraction domain.}
{This suggests a simple approach: always initialize with zero, $\hat{{\bf z}}_{t_i}^0=\boldsymbol{0}$, which we call \textbf{cold-start sampling} of a fixed-point flow.}
On the other hand, choosing the right initialization can improve the efficiency of finding the fixed point.
{
For instance, if the fixed points are correlated locally in $t$, the fixed-point estimate of the previous flow timestep can provide a good initialization, $\hat{{\bf z}}_{t_i}^0 = \hat{{\bf z}}_{t_{i-1}}^\star$.
We call this \textbf{warm-start sampling}.
}
{
The potential benefit of a better initialization is formally quantified under contractivity in \Cref{app:proof_convergence}.
}

{
From this initialization perspective, we can view conventional self-conditioned generation in \eqref{eq:sc_euler} as warm-start sampling that approximates fixed points with a single iteration, $\hat{{\bf z}}_{t_i}^\star\approx \hat{D}_{t_i}(\hat{{\bf x}}_{t_i},\hat{{\bf z}}_{t_{i-1}}^\star)$.}
This implies that the coupling of flow state and self-conditioning state in \eqref{eq:sc_euler} is not a defining property of self-conditioned flows, as it is merely a byproduct of warm starts, a sampling heuristic.
{We empirically find that cold-start sampling with sufficient fixed-point iterations can replace warm starts in ELF~\citep{hu2026elf} and LangFlow~\citep{chen2026langflow}.}

Taking this idea further, we find that it is possible to take a self-conditioned flow model and turn it into a self-conditioning-free model by distilling the self-conditioned denoiser $\hat{D}$ into the fixed-point denoiser $D^\star$.
This can be done, for instance, via the following fixed-point distillation loss:
\begin{equation}
    \mathcal{L}(D^\star)
    = \int_0^1 \mathbb{E}|D_t^\star(I_t) - {\bf z}^\star|^2\,{\rm d}t, \quad {\bf z}^\star = \hat{D}_t(I_t,{\bf z}^\star),
    \label{eq:fp-z}
\end{equation}
where the target ${\bf z}^\star$ is estimated by iterating ${\bf z}^{j+1}=\hat{D}_t(I_t,{\bf z}^j)$ from ${\bf z}^0=\boldsymbol{0}$.
Any other fixed-point distillation method can be used instead of \eqref{eq:fp-z}, such as consistency distillation~\citep{lin2026consistency}.
The resulting fixed-point velocity $b^\star$ depends only on the current time and flow state, so its generation can use the usual Euler scheme~\eqref{eq:euler}.

%% file: sections/distillation.tex
\vspace{-4pt}
\subsection{Fixed-point flow maps}

\label{sec:flowmap}
In \Cref{sec:fpflow}, we have formalized fixed-point flows, which express self-conditioned flows via ordinary velocity fields.
We leverage this insight to define their associated flow maps, which allows us to design few-step distillation methods for self-conditioned flow language models.
\begin{lavenderbox}
    \begin{restatable}{proposition}{fpflowmap}\label{prop:fpflowmap}
        On any interval where the fixed-point flow ODE \eqref{eq:fp_ode} has unique solutions, the \textbf{fixed-point flow map}, its solution operator $X_{s,t}^\star({\bf x}_s) = {\bf x}_t$, satisfies
        \begin{equation}\label{eq:fixedpointflowmap}
            X^\star_{s,t}({\bf x})={\bf x} + \int_s^t b_\tau^\star({\bf x}_\tau)\,{\rm d}\tau.
        \end{equation}
        Moreover, for $0\leq s\leq u\leq t< 1$,
        \begin{equation}\label{eq:fpfm-semigroup}
            X_{s,t}^\star = X_{u,t}^\star \circ X_{s,u}^\star.
        \end{equation}

    \end{restatable}
\end{lavenderbox}
A proof is in \Cref{app:proof_fpflowmap}.
The result shows that $X^\star$ is a valid flow map and admits conditions such as \eqref{eq:fpfm-semigroup} that can be leveraged to learn it through distillation from the velocity $b_t^\star$.
In order to learn the flow map, we parameterize it with a \textbf{two-time denoiser} $\delta_{s,t}$ defined as follows~\citep{lee2026flow,lu2026one}:
\begin{equation}\label{eq:twotimedenoiser}
    \delta_{s,t}({\bf x}) \coloneqq {\bf x} + (1-s) v_{s,t}({\bf x}),
\end{equation}
where $v_{s,t}$ is the average velocity from $s$ to $t$, defined as follows~\citep{boffi2026build}:
\begin{equation}
    v_{s,t}({\bf x}) \coloneqq \frac{X_{s,t}^\star({\bf x}) - {\bf x}}{t-s},\quad v_{t,t}({\bf x})\coloneqq b_t^\star({\bf x}).
\end{equation}
Analogous to the relationship between the single-time denoiser and velocity, $D_t^\star({\bf x}) = {\bf x} + (1-t)b_t^\star({\bf x})$, the two-time denoiser can be viewed as a single denoising step to the clean data domain.
Its outputs are known to lie in a low-dimensional manifold~\citep{lee2026flow,lu2026one}, making learning easier.
We now show:
\begin{lavenderbox}
    \begin{restatable}{proposition}{twotimedenoiser}\label{prop:twotimedenoiser}
        The two-time denoiser $\delta_{s,t}$ satisfies the following properties:
\begin{enumerate}[label=(\roman*),leftmargin=*]
\item For \(0\le s<t<1\), the flow map is recovered from \(\delta_{s,t}\) by
\begin{equation}\label{eq:fm-recovery}
X_{s,t}^\star({\bf x})
=
\frac{1-t}{1-s}{\bf x}
+
\frac{t-s}{1-s}\delta_{s,t}({\bf x}).
\end{equation}

\item On the diagonal,
\begin{equation}\label{eq:fm-diagonal}
\delta_{t,t}({\bf x})=D_t^\star({\bf x}).
\end{equation}

\item For \(0\le s<u<t<1\), the semigroup condition \eqref{eq:fpfm-semigroup} is equivalent to
\begin{equation}\label{eq:fm-semigroup}
\delta_{s,t}({\bf x})
=
\gamma \delta_{s,u}({\bf x})
+
(1-\gamma)\delta_{u,t}(X_{s,u}^\star({\bf x})),
\end{equation}
where $\gamma=\tfrac{(1-t)(u-s)}{(1-u)(t-s)}\in[0,1]$.
\end{enumerate}
\end{restatable}
\end{lavenderbox}
A proof is in \Cref{app:proof_twotimedenoiser}.
The result shows that learning the two-time denoiser yields the flow map, and that the two-time denoiser matches the fixed-point denoiser $D^\star$ \eqref{eq:fp_denoiser} on the diagonal $s=t$, while satisfying a self-consistency criterion off the diagonal.
These conditions lead to a learning objective:
\begin{lavenderbox}
    \begin{restatable}{proposition}{twotimedenoisertraining}\label{prop:twotimedenoisertraining}
        Consider the population objective
        \begin{equation}\label{eq:fp-distill}
            \mathcal L(\delta)
            =
            \mathbb E_{s,u,t,{\bf x}}
            |\delta_{s,t}({\bf x})-\mathsf{sg}(\mathcal T(\delta)_{s,u,t}({\bf x}))|^2
            +
            \mathbb E_{t,{\bf x}}
            |\delta_{t,t}({\bf x})-D_t^\star({\bf x})|^2
        \end{equation}
        where
        \begin{equation}
            \mathcal T(\delta)_{s,u,t}({\bf x})
            :=
            \gamma \delta_{s,u}({\bf x})
            +
            (1-\gamma)
            \delta_{u,t}\bigl(X_{s,u}^\star({\bf x})\bigr).
        \end{equation}
        
        Then the true two-time denoiser has zero loss.
        Conversely, any zero-loss solution satisfies the diagonal \eqref{eq:fm-diagonal} and semigroup \eqref{eq:fm-semigroup} conditions almost surely under the training distribution.
    \end{restatable}
\end{lavenderbox}
A proof is given in \Cref{app:proof_twotimedenoisertraining}.
The first loss term enforces the semigroup condition, and the second loss term anchors the diagonal to the fixed point $D_t^\star(I_t) = {\bf z}^\star = \hat{D}_t(I_t,{\bf z}^\star)$.
Evaluating the latter requires fixed points ${\bf z}^\star$, which we may make available in one of two ways:
an offline route that first distills the fixed-point denoiser $D^\star$ as a separate model via \eqref{eq:fp-z}, or an online route that performs two compressions in \eqref{eq:fp-z} and \eqref{eq:fp-distill} jointly, estimating ${\bf z}^\star$ with a few iterations of ${\bf z}^{j+1} = \hat{D}_t(I_t,{\bf z}^j)$ from ${\bf z}^0=\boldsymbol{0}$.
The online approach ultimately achieves $\delta_{t,t}\approx D_t^\star$ without training a separate model.

%% file: sections/experiments.tex
\section{Experiments}
\label{sec:exp}

Through our experiments, we aim to answer the following key questions:
\begin{enumerate}[label={\bf Q\arabic*}.,leftmargin=*]
    \item Do self-conditioned flow language models learn fixed-point iterations?
    \item {Does self-conditioned generation provide a better initialization for fixed-point iteration?}
    \item Can we train self-conditioning-free flows competitive with self-conditioned ones?
    \item Can we train a flow map language model that leverages self-conditioning?
\end{enumerate}

\paragraph{Setup.}
We conduct all experiments on the OpenWebText dataset~\citep{Gokaslan2019OpenWeb}, a standard corpus for language modeling, with a sequence length of 1024 tokens.
We assess generation quality with generative perplexity (gPPL)~\citep{dieleman2022continuous, sahoo2025duo} measured with pretrained GPT-2 Large~\citep{radford2019language}, together with average per-sample unigram entropy.
We regard a model as strong only when its generations attain low gPPL while remaining close to the data entropy of 5.44 nats \citep{lee2026flow}.
We use 256 samples for analysis (\Cref{subsec:analysis}) and 1024 for distillation (\Cref{subsec:distill}).
See \Cref{app:impl} for more experimental details.

\subsection{Analyzing self-conditioned flow language models} %
\label{subsec:analysis}
To answer whether self-conditioned denoisers induce fixed-point iterations (\textbf{Q1}) and whether the gains of self-conditioned sampling come from improved initialization of fixed-point iterations
(\textbf{Q2}), we study ELF \citep{hu2026elf} and LangFlow \citep{chen2026langflow} as representatives of self-conditioned flow language models on learned embeddings and one-hot token embeddings, respectively.
We draw gPPL-entropy frontier curves \citep{pynadath2025candi} by sweeping sampling parameter $\gamma$ for ELF and softmax temperature for LangFlow.

\paragraph{Self-conditioning induces a fixed-point iteration (Q1).}
We first ask whether self-conditioned denoisers learn fixed-point iteration~\eqref{eq:sc_map}.
Following~\citet{lin2026consistency}, we track relative distance to the fixed point, $|{\bf z}^j - {\bf z}^\star|/|{\bf z}^\star|$, which equals one at $j=0$ when cold-start sampling is used (${\bf z}^0={\bf 0}$).
To approximate ${\bf z}^\star$, we run 200 damped Picard iterations~\citep{lauriere2021numerical} with a fixed damping parameter $0.3$ (see \Cref{app:analysis_detail} for details).
{As shown in \Cref{fig:refinement-compact-draft}(a), the relative distance decays across fixed-point iterations $j$ for both models.}
This supports that the models have approximately learned a fixed point iteration as posited in \Cref{sec:formalization}.

We now ask whether taking more steps toward the fixed point improves generation quality.
{As illustrated in \Cref{fig:refinement-compact-draft}(b,c), gPPL decreases as the number of fixed-point iterations increases on both ELF and LangFlow, whereas entropy remains stable near the data entropy, showing that taking more steps toward the fixed point produces a better denoiser for generation.}
{
In \Cref{app:qualitative_refinement}, we further show qualitatively that iterative refinement improves the consistency and coherence of the model's predictions.
}

\input{figures/refinement_compact_draft}

\begin{figure}[t!]
  \centering
  \resizebox{0.75\linewidth}{!}{%
      \begin{subfigure}[t]{0.43\linewidth}
        \centering
        \includegraphics[width=\linewidth,keepaspectratio]{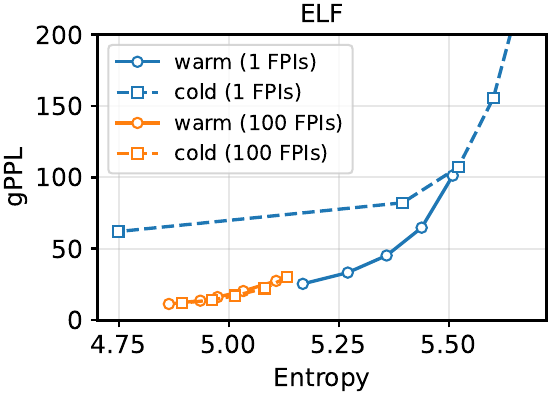}
      \end{subfigure}\hfill
      \hspace{0.5cm}
      \begin{subfigure}[t]{0.43\linewidth}
        \centering
        \includegraphics[width=\linewidth,keepaspectratio]{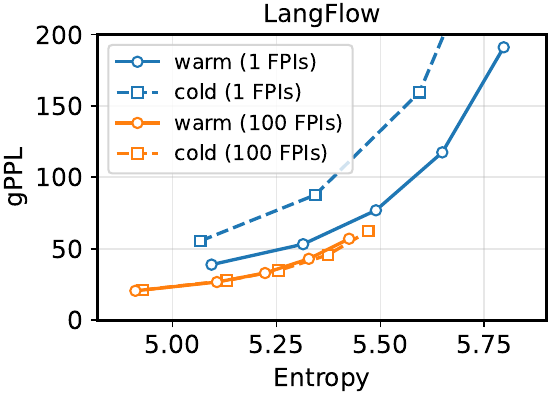}
      \end{subfigure}
  }
  \vspace{-0.05in}
  \caption{
  Warm-start and cold-start sampling with 1 and 100 fixed-point iterations.
  }
  \label{fig:warm_vs_cold_j8}
  \vspace{-0.4cm}
\end{figure}

\paragraph{{Self-conditioned generation provides a better initialization for fixed-point iteration (Q2).}}
In \Cref{sec:fpflow}, we viewed conventional generation of self-conditioned flows as a warm-start sampling scheme that attempts to use a better initialization for the fixed-point iteration.
Accordingly, we conjectured that cold-start sampling would reach a similar performance if sufficient iterations are used as it would eventually find the fixed point regardless of initialization.
To validate this, we draw the gPPL–entropy frontier of each sampling scheme in~\Cref{fig:warm_vs_cold_j8}.
{When we use a single fixed-point iteration, warm-start outperforms cold-start sampling, suggesting that warm-start provides an effective initialization for fixed-point iteration.}
With 100 iterations, however, warm- and cold-start attain the same frontier, indicating that the choice of initialization has a negligible effect on generation once enough fixed-point iterations are run.

\subsection{Distilling self-conditioning into fixed-point flow maps} %
\label{subsec:distill}
To test self-conditioning-free distillation (\textbf{Q3}) and self-conditioned flow map distillation (\textbf{Q4}), we train ELF~\citep{hu2026elf} as a teacher model, replacing its encoder with GPT-2 Large \citep{radford2019language} for tokenizer-matched comparison with baselines.

\begin{table}[t!]
  \centering
  \footnotesize
  \setlength{\tabcolsep}{1.5pt}
  \caption{Comparison with few-step language models on OpenWebText. Offline FMLM$^\star$ outperforms all baselines that preserve data-level entropy (5.44 nats).}
  \vspace{-0.1in}
  \label{tab:onestep}
  \begin{tabular*}{\linewidth}{@{\extracolsep{\fill}}cc|ccccccc|cc@{}}
    \toprule
    \multirow{2}{*}{Step} & \multirow{2}{*}{Metric} & \multicolumn{2}{c}{Duo} & \multicolumn{2}{c}{MDLM} & \multirow{2}{*}{FMLM} & \multicolumn{2}{c|}{DFM} & \multicolumn{2}{c}{FMLM$^\star$} \\
    \cmidrule(lr){3-4} \cmidrule(lr){5-6} \cmidrule(lr){8-9} \cmidrule(lr){10-11}
     & & DCD & Di4C & SDTT & Di4C & & PSD & ESD & Offline & Online \\
    \midrule
    \multirow{2}{*}{1} & gPPL $\downarrow$ & 5743.29 & 370.51 & 1260.86 & 1298.80 & 168.30 & 180.29 & 5.33 & \textbf{113.91} & 117.89 \\
                               & entropy & 6.02 & \textcolor{red}{3.92} & 5.26 & 5.29 & 5.17 & 4.91 & \textcolor{red}{0.26} & 5.37 & 5.39 \\
    \midrule
    \multirow{2}{*}{2} & gPPL $\downarrow$ & 891.16 & 210.22 & 877.22 & 758.23 & 133.29 & 152.83 & 108.91 & \textbf{93.66} & 98.96 \\
                               & entropy & 5.41 & 4.63 & 5.34 & 5.35 & 5.25 & 5.03 & 5.15 & 5.46 & 5.47 \\
    \midrule
    \multirow{2}{*}{4} & gPPL $\downarrow$ & 250.86 & 154.67 & 339.73 & 239.27 & 111.31 & 122.32 & 77.08 & \textbf{73.86} & 80.26 \\
                               & entropy & 5.37 & 4.85 & 5.38 & 5.40 & 5.26 & 5.10 & 5.27 & 5.35 & 5.44 \\
    \bottomrule
  \end{tabular*}
  \vspace{-1em} %
\end{table}

\begin{wrapfigure}[12]{r}{0.33\textwidth}
    \centering
    \includegraphics[width=\linewidth,keepaspectratio]{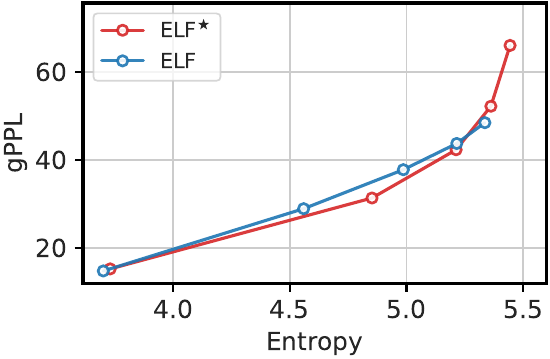}
    \vspace{-0.7cm}
    \caption{Comparison between ELF$^\star$ and its teacher model ELF.
    }
    \label{fig:coldelf}
\end{wrapfigure}

\paragraph{Self-conditioning is removable (Q3).}
We find that, given a self-conditioned denoiser $\hat D$ (ELF in our case), we can learn its associated fixed-point denoiser $D^\star$ \eqref{eq:fp_denoiser} through fixed-point distillation, which yields fixed-point velocity \eqref{eq:fp_velocity}, a self-conditioning-free model.
Herein, we consider CDEQ~\citep{lin2026consistency}, an existing fixed-point distillation method.
Using the method, we distill the ELF teacher into a self-conditioning-free model \textbf{ELF$^\star$}, with
implementation details provided in \Cref{app:impl}.
We compare 32-step generation performance of ELF$^\star$ against 32-step generation of the teacher with gPPL-entropy frontier curves.
As in \Cref{fig:coldelf}, ELF$^\star$ matches the teacher's frontier, confirming that self-conditioning is removable via fixed-point distillation without degrading quality.

\begin{wrapfigure}[14]{r}{0.33\textwidth}
    \centering
    \vspace{-0.6cm}
    \includegraphics[width=\linewidth,keepaspectratio]{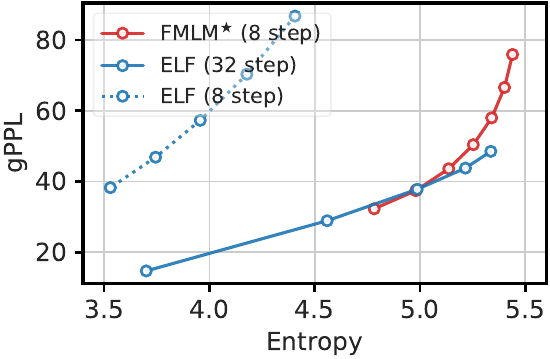}
    \vspace{-0.6cm}
    \caption{Comparison between FMLM$^\star$ and teacher model ELF.}
    \label{fig:scfmlm-staged}
\end{wrapfigure}
\paragraph{Self-conditioning is distillable into a flow map (Q4).}
Finally, we ask whether self-conditioning can be leveraged to train a few-step flow map.
To this end, we distill the self-conditioned ELF teacher into a self-conditioned flow map language model \textbf{FMLM}$^\star$, parameterized by the two-time denoiser $\delta_{s,t}$~\eqref{eq:twotimedenoiser}.
Following the offline route proposed in \Cref{sec:flowmap}, we take the fixed-point denoiser ELF$^\star$ as teacher and learn its associated fixed-point flow map with the semigroup distillation objective~\eqref{eq:fp-distill}.
As illustrated in \Cref{fig:scfmlm-staged}, FMLM$^\star$ with 8 generation steps approaches the 32-step frontier of the self-conditioned teacher ELF.

{We compare FMLM$^\star$ with discrete diffusion baselines Duo with DCD~\citep{sahoo2025duo}, MDLM with SDTT~\citep{deschenaux2024beyond}, both with Di4C~\citep{hayakawa2024distillation}, and continuous flow map models FMLM~\citep{lee2026flow} and DFM~\citep{potaptchik2026discrete}.}
{
In \Cref{tab:onestep}, offline FMLM$^\star$ achieves the best gPPL among baselines preserving data-level entropy under matched sampling budgets.
Online distillation with 8 FPIs avoids training a separate fixed-point denoiser, achieving competitive quality while reducing training time from 82 to 23 hours.
Additional FPI results and generation samples appear in Appendices~\ref{app:results} and~\ref{app:generation_samples}.}

%% file: figures/refinement_compact_draft.tex
\begin{figure}[t!]
  \centering
  \includegraphics[width=\linewidth]{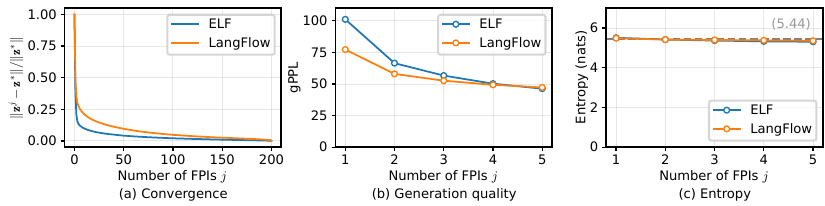}
  \par\vspace{-4pt} %
  \caption{Convergence and generation quality across fixed-point iterations for ELF and LangFlow.
  (a) Relative distance to the fixed point.
  (b) gPPL and (c) entropy with 32 sampling steps for ELF and 64 for LangFlow.
  The dashed gray line indicates data entropy (5.44 nats).}
  \label{fig:refinement-compact-draft}
  \vspace{-6pt} %
\end{figure}

%% file: sections/related_work.tex
\section{Related Work}
\label{sec:related_work}

{
\paragraph{Self-conditioning in continuous language models.}
Self-conditioning reuses denoising predictions as model inputs~\citep{chen2022analog} and has been applied to continuous text diffusion~\citep{strudel2022self,dieleman2022continuous}.
\citet{gulrajani2023likelihood} interpret the self-conditioned denoiser as a fixed point of prediction updates.
We develop this perspective by analyzing convergence under contractivity and separating prediction refinement from flow evolution to define fixed-point flows and their flow maps.

}

{
\paragraph{Fixed-point computation in neural architectures.}
Fixed-point iterations enable neural networks to refine their representations repeatedly using shared parameters.
\citet{bai2019deep} define hidden representations as fixed points and compute them via root-finding, while \citet{movahedi2026fixed} use fixed-point convergence to adapt computation in looped transformers.
For generative modeling, FPDM~\citep{bai2024fixedpoint} and FPMGM~\citep{miele2026fixedpointmaskedgenerativemodeling} incorporate such iterations into diffusion and masked generative denoisers to improve generation quality and computational efficiency.
We use the fixed-point iteration perspective to formalize and analyze self-conditioning in flow language models.
This analysis allows us to define fixed-point flows and their associated flow maps, enabling flow map distillation of self-conditioned models.

}

{
\paragraph{Flow map language models.}
Flow map language models build on the flow map framework of \citet{boffi2026build} to enable one- and few-step text generation.
FMLM~\citep{lee2026flow} formulates language generation as a continuous flow and distills its associated flow map.
CFM~\citep{roos2026categorical} develops a simplex-constrained parameterization for categorical data and learns flow maps through self-distillation, including an endpoint-consistency objective.
DFM~\citep{potaptchik2026discrete} further studies flow map training objectives adapted to the geometry of the probability simplex.
Integrating self-conditioning into flow map learning remains challenging because the denoiser depends on an additional prediction state that evolves across sampling steps.
}

%% file: sections/conclusion.tex
\vspace{-6pt}
\section{Conclusion}
\vspace{-3pt}
\label{sec:conclusion}
{
We demonstrated that self-conditioned flow language models implicitly learn a fixed-point iteration that refines their predictions.
Under contractivity, the iteration converges to a fixed point; empirically, additional iterations improve generation quality.
Based on this viewpoint, we proposed fixed-point flows, a two-dimensional class of self-conditioned flows, and interpreted conventional self-conditioned sampling as a single-step warm-started fixed-point iteration.
Building on this framework, we distilled both fixed-point iterations and flow to obtain FMLM$^\star$, which achieves state-of-the-art one- and few-step generation on OpenWebText under matched sampling budgets.
While these results demonstrate the effectiveness of our approach, our flow map learning requires a separate teacher, motivating further research into self-distillation that jointly learns the self-conditioned, fixed-point, and two-time denoisers from data supervision.
Our evaluation is also limited to text, leaving the applicability of the framework to other modalities an open question.}

\ificlrfinal
\section*{Acknowledgment}
\draft{This work was supported in part by the National Research Foundation of Korea (RS-2024-00351212, RS-2024-00436165), the Institute of Information \& Communications Technology Planning \& Evaluation (IITP) (RS-2024-00509279, RS-2022-II220926, RS-2022-II220959, RS-2019-II190075), the InnoCORE program of the Ministry of Science and ICT (AI Meta-Scientist, N10260110), and the ``HPC support'' project, funded by the Korean government (MSIT).}
The authors would like to thank Sander Dieleman for valuable discussions.
\fi

%% file: sections/appendix_proofs.tex
\section{{Proofs and Theoretical Results}}
\label{app:proofs}

\subsection{Proof of Proposition~\ref{prop:optimality}}
\label{app:proof_optimality}

\begin{lavenderbox}
    \optimality*
\end{lavenderbox}

\begin{proof}
Fix \(t\) for which the conditional independence assumption holds, and write
\[
D\coloneqq D_t(I_t),
\qquad
U\coloneqq \hat D_t(I_t,{\bf z}).
\]
Then
\[
{\bf x}_1-U=({\bf x}_1-D)+(D-U).
\]
Squaring and taking expectations gives
\[
\mathbb{E}|{\bf x}_1-U|^2
=
\mathbb{E}|{\bf x}_1-D|^2
+
\mathbb{E}|U-D|^2
+
2\mathbb{E}\langle {\bf x}_1-D,D-U\rangle.
\]
The cross term vanishes. Indeed,
\[
\mathbb{E}[{\bf x}_1-D\mid I_t,{\bf z}]
=
\mathbb{E}[{\bf x}_1\mid I_t,{\bf z}]-D_t(I_t).
\]
By the conditional independence assumption,
\[
\mathbb{E}[{\bf x}_1\mid I_t,{\bf z}]
=
\mathbb{E}[{\bf x}_1\mid I_t]
=
D_t(I_t).
\]
Hence
\[
\mathbb{E}[{\bf x}_1-D\mid I_t,{\bf z}]=0,
\]
and therefore
\[
\mathbb{E}\langle {\bf x}_1-D,D-U\rangle=0.
\]
Thus
\[
\mathbb{E}|\hat D_t(I_t,{\bf z})-{\bf x}_1|^2
=
\mathbb{E}|{\bf x}_1-D_t(I_t)|^2
+
\mathbb{E}|\hat D_t(I_t,{\bf z})-D_t(I_t)|^2.
\]
The first term does not depend on \(\hat D_t\). The second term is minimized exactly when
\[
\hat D_t(I_t,{\bf z})=D_t(I_t)
\]
almost surely under the law of \((I_t,{\bf z})\). Integrating over \(t\) gives the claim.
\end{proof}

\subsection{Proof of approximate contractivity}
\label{app:proof_contractivity}

We show that learning the self-conditioned objective leads to an approximate notion of contractivity.

\begin{lavenderbox}
    \begin{definition}[\bf A simplified self-conditioned regression.]
        Let $O\subseteq\mathbb{R}^d$ be a nonempty and closed set, let $Z\sim \mu$ with $\mu(O) = 1$, let $X\in\mathbb{R}^d$ be fixed, and let $f:\mathbb{R}^d\to\mathbb{R}^d$ be a continuous function.
        For some $0 < p < 1$, we define the following self-conditioned regression objective:
        \begin{equation}\label{eq:simplified_self_conditioned_regression}
            \mathcal{L}(f)\coloneqq p\mathbb{E}|f(Z) - X|^2 + (1-p)\mathbb{E}|f(f(Z)) - X|^2.
        \end{equation}
        \end{definition}
\end{lavenderbox}
This is a simplification of our problem setting in \Cref{sec:sc_flm}.
Taking $p=1/2$, $f = \hat{D}_t(I_t,\cdot)$, and $X = D_t({\bf x})$ (the Bayes-optimal prediction target; \Cref{prop:optimality}), for fixed $t\in[0,1]$ and ${\bf x}$, we recover a practical learning scenario for self-conditioned language models similar to \eqref{eq:sc_loss_decomposed}.

Let us fix
\begin{itemize}[leftmargin=*]
    \item a scale $w>0$,
    \item a factor $0\leq\eta<1$,
    \item a tolerance $r>0$,
    \item and a failure probability $0<\varepsilon<1$.
\end{itemize}
Let us assume:
\begin{equation}\label{eq:approx_contractivity_assumption}
    \bar{B}(X,r) \subseteq O,\quad 2r\leq \eta w,\quad \mathcal{L}(f)\leq \min\{p, 1-p\}\varepsilon r^2.
\end{equation}
Define
\begin{equation}
    A \coloneqq \{{\bf z}\in O:|f({\bf z}) - X|\leq r\}.
\end{equation}

The following lemmas show that a sufficiently small loss yields a closed, high-probability set $A$ that is contractive with factor $\eta$ above the scale $w$, and approximately forward-invariant under the training distribution with leakage at most $\varepsilon$.
Intuitively, the first term in the self-conditioned loss makes the set $A$ high-probability under $\mu$, whereas the second term makes $f(Z)$ remain in $A$ for most $Z\sim \mu$.
\begin{lavenderbox}
    \begin{lemma}[\bf Closedness and approximate forward invariance]\label{lem:approx_contractivity_1}
        $A$ is a closed subset of $\mathbb{R}^d$.
        Furthermore, $f(A)\subseteq O$, and it holds that
        \begin{equation}
            \mu(A\cap f^{-1}(A)) \geq 1-\varepsilon.
        \end{equation}
        In particular, $\mu(A)\geq 1-\varepsilon$ and
        \begin{equation}
            \mu(\{{\bf z}\in A: f({\bf z})\notin A\}) \leq \varepsilon.
        \end{equation}
    \end{lemma}
\end{lavenderbox}

\begin{proof}
    The closed Euclidean ball $\bar{B}(X, r)$ is closed.
    Since $f$ is continuous,
    \[
    f^{-1}(\bar{B}(X, r))
    \]
    is closed.
    Therefore,
    \[
    A = O\cap f^{-1}(\bar{B}(X, r))
    \]
    is closed because $O$ is closed.

    For every ${\bf z}\in A$, we have, by definition,
    \[
    |f({\bf z}) - X| \leq r.
    \]
    Hence
    \[
    f({\bf z})\in \bar{B}(X, r).
    \]
    By assumption \eqref{eq:approx_contractivity_assumption},
    \[
    f(A)\subseteq O.
    \]

    Define the initial and self-conditioned errors
    \[
    e_1({\bf z})\coloneqq |f({\bf z})-X|,\quad e_2({\bf z})\coloneqq |f(f({\bf z}))-X|.
    \]
    and define
    \[
    G \coloneqq \{{\bf z}\in O: e_1({\bf z})\leq r \text{ and } e_2({\bf z})\leq r\}.
    \]
    The condition $e_1({\bf z})\leq r$ is exactly ${\bf z}\in A$.

    Moreover, as we have seen, $e_1({\bf z})\leq r$ implies $f({\bf z})\in O$.
    Therefore,
    \[
    e_2({\bf z})\leq r
    \]
    is equivalent to
    \[
    f({\bf z})\in A.
    \]
    Consequently,
    \[
    G = A\cap f^{-1}(A).
    \]

    On the complement $G^c$, at least one of the inequalities
    \[
    e_1({\bf z}) \leq r,\quad e_2({\bf z}) \leq r
    \]
    fails.

    If $e_1({\bf z}) > r$, then
    \[
    p e_1({\bf z})^2 > pr^2 \geq \min\{p, 1-p\}r^2.
    \]
    If $e_2({\bf z}) > r$, then
    \[
    (1-p) e_2({\bf z})^2 > (1-p)r^2 \geq \min\{p, 1-p\}r^2.
    \]
    Thus, for every ${\bf z}\notin G$,
    \[
    pe_1({\bf z})^2 + (1-p)e_2({\bf z})^2 \geq \min\{p, 1-p\}r^2.
    \]
    Equivalently,
    \[
    \boldsymbol{1}({\bf z}\notin G) \leq \frac{pe_1({\bf z})^2 + (1-p)e_2({\bf z})^2}{\min\{p, 1-p\}r^2}.
    \]
    Taking expectation with respect to $Z\sim \mu$,
    \[
    \mu(G^c)\leq \frac{\mathbb{E}[pe_1({\bf z})^2 + (1-p)e_2({\bf z})^2]}{\min\{p, 1-p\}r^2} = \frac{\mathcal{L}(f)}{\min\{p, 1-p\}r^2}.
    \]
    By assumption \eqref{eq:approx_contractivity_assumption},
    \[
    \mu(G^c)\leq \varepsilon.
    \]
    Therefore,
    \[
    \mu(G)\geq 1 - \varepsilon.
    \]
    Since $G = A\cap f^{-1}(A)$,
    \[
    \mu(A\cap f^{-1}(A))\geq 1 - \varepsilon.
    \]
    This implies $\mu(A)\geq 1 - \varepsilon$ because $A\cap f^{-1}(A)\subseteq A$.

    It also implies
    \[
    \mu(A\setminus f^{-1}(A))\leq \varepsilon,
    \]
    because the event
    \[
    A\setminus f^{-1}(A) = \{{\bf z}\in A:f({\bf z})\notin A\}
    \]
    is contained in $G^c$.
\end{proof}

\begin{lavenderbox}
    \begin{lemma}[\bf Approximate contractivity]\label{lem:approx_contractivity_2}
        For every ${\bf z},{\bf z}'\in A$ satisfying $|{\bf z} - {\bf z}'| \geq w$, one has
        \begin{equation}
            |f({\bf z}) - f({\bf z}')|\leq \eta |{\bf z} -{\bf z}'|.
        \end{equation}
    \end{lemma}
\end{lavenderbox}

\begin{proof}
    Take arbitrary ${\bf z},{\bf z}'\in A$.
    By definition of $A$,
    \[
    |f({\bf z}) - X| \leq r,\quad |f({\bf z}') - X| \leq r.
    \]
    By the triangle inequality,
    \[
    |f({\bf z}) - f({\bf z}')| \leq |f({\bf z}) - X| + |f({\bf z}') - X| \leq 2r.
    \]
    By assumption,
    \[
    2r \leq \eta w.
    \]
    Therefore, whenever $|{\bf z} - {\bf z}'| \geq w$,
    \[
    |f({\bf z}) - f({\bf z}')| \leq 2r \leq \eta w \leq \eta|{\bf z} - {\bf z}'|.
    \]
    Hence
    \[
    |f({\bf z}) - f({\bf z}')|\leq \eta |{\bf z} -{\bf z}'|
    \]
    for every pair ${\bf z},{\bf z}'\in A$ separated by at least $w$.
\end{proof}

\subsection{Proof of Proposition~\ref{prop:learnedfixedpoint}}
\label{app:proof_learnedfixedpoint}

We first show a useful lemma.
\begin{lavenderbox}
    \begin{lemma}\label{lem:geometric_convergence}
        Let $f:\mathbb{R}^d\to\mathbb{R}^d$ be a contraction on a nonempty closed set $O\subseteq\mathbb{R}^d$ with factor $0\leq\eta<1$.
        Then, there is a unique fixed point ${\bf z}^\star\in O$, to which the iteration
        \[
        {\bf z}^{j+1} = f({\bf z}^j)
        \]
        from any ${\bf z}^0\in O$ converges exponentially.
        Furthermore, the following holds for any finite $j\geq k\geq 0$:
        \[
        |{\bf z}^j - {\bf z}^k| \leq \frac{\eta^k-\eta^j}{1-\eta} |{\bf z}^1 - {\bf z}^0|.
        \]
        Similarly, the following holds:
        \[
        |{\bf z}^\star - {\bf z}^k| \leq \frac{\eta^k}{1-\eta} |{\bf z}^1 - {\bf z}^0|.
        \]
    \end{lemma}
\end{lavenderbox}
\begin{proof}
Because $O$ is a closed subset of $\mathbb{R}^d$, it is complete.
The map \(f\) maps \(O\) into itself and is a contraction on \(O\).
Therefore, by the Banach fixed-point theorem, \(f\) has a unique fixed point \({\bf z}^\star\in O\), and the iterates converge to it.

A geometric convergence bound follows directly from the contraction property:
\[
|{\bf z}^{j+1}-{\bf z}^\star|
=
|f({\bf z}^j) - f({\bf z}^\star)|
\le
\eta |{\bf z}^j-{\bf z}^\star|.
\]
Repeating this inequality gives
\begin{equation}\label{eq:geometric_convergence}
|{\bf z}^j-{\bf z}^\star|
\le
\eta^j|{\bf z}^0-{\bf z}^\star|.
\end{equation}

For $j\geq k\geq 0$, telescoping gives
\[
{\bf z}^j - {\bf z}^k = \sum_{i=k}^{j-1}({\bf z}^{i+1} - {\bf z}^i).
\]
Taking norms and applying the triangle inequality,
\begin{equation}\label{eq:telescoping_norm}
    |{\bf z}^j - {\bf z}^k| = \left|\sum_{i=k}^{j-1}({\bf z}^{i+1} - {\bf z}^i)\right| \leq \sum_{i=k}^{j-1} |{\bf z}^{i+1} - {\bf z}^i|.
\end{equation}
Each increment can be bounded using the contraction property:
\[
|{\bf z}^{i+1} - {\bf z}^i| = |f({\bf z}^i) - f({\bf z}^{i-1})| \leq \eta |{\bf z}^i - {\bf z}^{i-1}|.
\]
Repeating this inequality gives
\[
|{\bf z}^{i+1} - {\bf z}^i| \leq \eta^i |{\bf z}^1 - {\bf z}^0|.
\]
Applying this in \eqref{eq:telescoping_norm},
\[
|{\bf z}^j - {\bf z}^k| \leq \sum_{i=k}^{j-1} \eta^i |{\bf z}^1 - {\bf z}^0|.
\]
The geometric sum is
\[
\sum_{i=k}^{j-1} \eta^i = \frac{\eta^k-\eta^j}{1-\eta}.
\]
Thus,
\[
|{\bf z}^j - {\bf z}^k| \leq \frac{\eta^k-\eta^j}{1-\eta} |{\bf z}^1 - {\bf z}^0|.
\]
Finally, since ${\bf z}^j\to{\bf z}^\star$, continuity of the norm gives
\[
|{\bf z}^\star-{\bf z}^k|
=
\lim_{j\to\infty}|{\bf z}^j-{\bf z}^k|.
\]
Therefore,
\[
|{\bf z}^\star-{\bf z}^k|
\leq
\lim_{j\to\infty}
\frac{\eta^k-\eta^j}{1-\eta}
|{\bf z}^1-{\bf z}^0|
=
\frac{\eta^k}{1-\eta}
|{\bf z}^1-{\bf z}^0|.
\]
\end{proof}

\begin{lavenderbox}
    \learnedfixedpoint*
\end{lavenderbox}

\begin{proof}
Since $\hat{D}_t({\bf x},\cdot)$ is a contraction on $O$, by applying \Cref{lem:geometric_convergence} to it, we have that ${\bf z}^j$ converges exponentially to a unique fixed point ${\bf z}^\star\in O$.
Furthermore, we have
\begin{equation}\label{eq:difference_bound}
    |{\bf z}^j - {\bf z}^k| \leq \frac{\eta^k-\eta^j}{1-\eta} |{\bf z}^1 - {\bf z}^0|,\quad |{\bf z}^\star - {\bf z}^k| \leq \frac{\eta^k}{1-\eta} |{\bf z}^1 - {\bf z}^0|.
\end{equation}
By the triangle inequality,
\[
|{\bf z}^j - D_t({\bf x})| \leq |{\bf z}^j - \hat{D}_t({\bf x},{\bf z}_0)| + |\hat{D}_t({\bf x},{\bf z}_0) - D_t({\bf x})|.
\]
By the definition of the iteration,
\[
|{\bf z}^j - D_t({\bf x})| \leq |{\bf z}^j - {\bf z}^1| + |\hat{D}_t({\bf x},{\bf z}_0) - D_t({\bf x})|.
\]

Using the first part of \eqref{eq:difference_bound} with $k=1$,
\[
|{\bf z}^j - D_t({\bf x})| \leq |\hat{D}_t({\bf x},{\bf z}_0) - D_t({\bf x})| + \frac{\eta-\eta^j}{1-\eta} |{\bf z}^1 - {\bf z}^0|.
\]
Applying the same argument, with the second part of \eqref{eq:difference_bound},
\[
|{\bf z}^\star - D_t({\bf x})| \leq |\hat{D}_t({\bf x},{\bf z}_0) - D_t({\bf x})| + \frac{\eta}{1-\eta} |{\bf z}^1 - {\bf z}^0|.
\]
\end{proof}

\subsection{Proof of Proposition~\ref{prop:fpflow}}
\label{app:proof_fpflow}

\begin{lavenderbox}
    \fpflow*
\end{lavenderbox}

\begin{proof}
By definition, \(D_t^\star({\bf x})\) is obtained by solving the inner self-conditioning fixed-point problem at the pair \((t,{\bf x})\). Hence, after taking the fixed point, there is no remaining conditioning variable \({\bf z}\). Therefore
\[
b_t^\star({\bf x})
=
\frac{D_t^\star({\bf x})-{\bf x}}{1-t}
\]
is a function only of \(t\) and \({\bf x}\).

The sampling dynamics are therefore the ordinary ODE
\[
\dot {\bf x}_t=b_t^\star({\bf x}_t).
\]

When the fixed point matches the Bayes-optimal denoiser, $D^\star=D$, then by \eqref{eq:flow_velocity}, we have that
\[
b_t^\star({\bf x}) = \frac{D_t({\bf x})-{\bf x}}{1-t} = b_t({\bf x}).
\]
Therefore, the fixed-point velocity recovers the true velocity.
\end{proof}

\subsection{{Iteration bound under contractivity}}
\label{app:proof_convergence}

{
We formally state and prove the iteration bound underlying the initialization discussion in \Cref{sec:fpflow}.
\par}

\begingroup

\begin{lavenderbox}
    \begin{proposition}\label{prop:convergence}
        In the setting of \Cref{prop:learnedfixedpoint}, let $\varepsilon>0$.
        If $0<\eta<1$ and $|{\bf z}^0-{\bf z}^\star|>\varepsilon$, then any integer
        \begin{equation}
            j\geq \left\lceil\frac{\log\bigl(|{\bf z}^0-{\bf z}^\star|/\varepsilon\bigr)}{\log(1/\eta)}\right\rceil
        \end{equation}
        guarantees $|{\bf z}^j-{\bf z}^\star|\leq\varepsilon$.
        If $|{\bf z}^0-{\bf z}^\star|\leq\varepsilon$, then $j=0$ suffices.
        If $\eta=0$, then $j=1$ suffices.
    \end{proposition}
\end{lavenderbox}

\begin{proof}
If the initial error is at most $\varepsilon$, no iteration is needed.
If $\eta=0$, contractivity gives
\[
|{\bf z}^1-{\bf z}^\star|
=|f({\bf z}^0)-f({\bf z}^\star)|=0,
\]
where $f=\hat{D}_t({\bf x},\cdot)$.
Otherwise, $0<\eta<1$ and $|{\bf z}^0-{\bf z}^\star|>\varepsilon$.
The geometric convergence bound \eqref{eq:geometric_convergence} in \Cref{lem:geometric_convergence} gives
\[
|{\bf z}^j-{\bf z}^\star|\leq\eta^j|{\bf z}^0-{\bf z}^\star|.
\]
It therefore suffices to require
\[
\eta^j|{\bf z}^0-{\bf z}^\star|\leq\varepsilon.
\]
Taking logarithms and using $\log\eta<0$ gives
\[
j\geq\frac{\log\bigl(|{\bf z}^0-{\bf z}^\star|/\varepsilon\bigr)}{\log(1/\eta)}.
\]
Taking the ceiling yields the stated integer bound.
\end{proof}
\endgroup

\subsection{Proof of Proposition~\ref{prop:fpflowmap}}
\label{app:proof_fpflowmap}

\begin{lavenderbox}
    \fpflowmap*
\end{lavenderbox}

\begin{proof}
    Since \eqref{eq:fp_ode} is an ODE, it defines the flow map \(X_{s,t}^\star({\bf x}_s)={\bf x}_t\) wherever the ODE solution is unique. Uniqueness also gives the composition law: evolving from \(s\) to \(t\) is the same as first evolving from \(s\) to \(u\), and then from \(u\) to \(t\). Thus
    \[
    X_{s,t}^\star
    =
    X_{u,t}^\star\circ X_{s,u}^\star.
    \]
\end{proof}

\subsection{Proof of Proposition~\ref{prop:twotimedenoiser}}
\label{app:proof_twotimedenoiser}

\begin{lavenderbox}
\twotimedenoiser*
\end{lavenderbox}
\begin{proof} 
    First, by definition,
    \[
    \delta_{s,t}({\bf x})
    =
    {\bf x}+(1-s)v_{s,t}({\bf x}).
    \]
    Therefore
    \[
    v_{s,t}({\bf x})
    =
    \frac{\delta_{s,t}({\bf x})-{\bf x}}{1-s}.
    \]
    Substituting this into
    \[
    X_{s,t}^\star({\bf x})
    =
    {\bf x}+(t-s)v_{s,t}({\bf x})
    \]
    gives
    \[
    X_{s,t}^\star({\bf x})
    =
    {\bf x}+\frac{t-s}{1-s}\bigl(\delta_{s,t}({\bf x})-{\bf x}\bigr).
    \]
    Rearranging gives
    \[
    X_{s,t}^\star({\bf x})
    =
    \frac{1-t}{1-s}{\bf x}
    +
    \frac{t-s}{1-s}\delta_{s,t}({\bf x}),
    \]
    which is \Cref{eq:fm-recovery}.

    For the diagonal identity,
    \[
    \delta_{t,t}({\bf x})
    =
    {\bf x}+(1-t)b_t^\star({\bf x}).
    \]
    Since
    \[
    b_t^\star({\bf x})=\frac{D_t^\star({\bf x})-{\bf x}}{1-t},
    \]
    we get
    \[
    \delta_{t,t}({\bf x})=D_t^\star({\bf x}),
    \]
    which is \Cref{eq:fm-diagonal}.
    
    It remains to prove the semigroup identity. Let
    \[
    {\bf z}:=X_{s,u}^\star({\bf x})
    =
    \frac{1-u}{1-s}{\bf x}
    +
    \frac{u-s}{1-s}\delta_{s,u}({\bf x}).
    \]
    Using \Cref{eq:fm-recovery},
    \[
    X_{s,t}^\star({\bf x})
    =
    \frac{1-t}{1-s}{\bf x}
    +
    \frac{t-s}{1-s}\delta_{s,t}({\bf x}),
    \]
    \[
    X_{u,t}^\star({\bf z})
    =
    \frac{1-t}{1-u}{\bf z}
    +
    \frac{t-u}{1-u}\delta_{u,t}({\bf z}).
    \]
    Since the semigroup property~\Cref{eq:fpfm-semigroup} holds,
    \[
    X_{s,t}^\star({\bf x})=X_{u,t}^\star(X_{s,u}^\star({\bf x}))=X_{u,t}^\star({\bf z}).
    \]
    Substituting the expression for \({\bf z}=X_{s,u}^\star({\bf x})\) into the right-hand side gives
    \[
    \begin{aligned}
    X_{s,t}^\star({\bf x})
    &= \frac{1-t}{1-u}\left(\frac{1-u}{1-s}{\bf x}+\frac{u-s}{1-s}\delta_{s,u}({\bf x})\right) 
    +
    \frac{t-u}{1-u}\delta_{u,t}({\bf z}) \\
    & =
    \frac{1-t}{1-s}{\bf x}
    +
    \frac{(1-t)(u-s)}{(1-u)(1-s)}\delta_{s,u}({\bf x})
    +
    \frac{t-u}{1-u}\delta_{u,t}({\bf z}).
    \end{aligned}
    \]
    and canceling the common \({\bf x}\)-term gives
    \[
    \frac{t-s}{1-s}\delta_{s,t}({\bf x})
    =
    \frac{(1-t)(u-s)}{(1-u)(1-s)}\delta_{s,u}({\bf x})
    +
    \frac{t-u}{1-u}\delta_{u,t}({\bf z}).
    \]
    Multiplying by \((1-s)/(t-s)\), we obtain
    \[
    \delta_{s,t}({\bf x})
    =
    \frac{(1-t)(u-s)}{(1-u)(t-s)}
    \delta_{s,u}({\bf x})
    +
    \frac{(1-s)(t-u)}{(1-u)(t-s)}
    \delta_{u,t}({\bf z}).
    \]
    Define
    \[
    \gamma
    =
    \frac{(1-t)(u-s)}{(1-u)(t-s)}.
    \]
    Then
    \[
    1-\gamma
    =
    \frac{(1-s)(t-u)}{(1-u)(t-s)}.
    \]
    Therefore
    \[
    \delta_{s,t}({\bf x})
    =
    \gamma \delta_{s,u}({\bf x})
    +
    (1-\gamma)\delta_{u,t}\bigl(X_{s,u}^\star({\bf x})\bigr).
    \]
    
    Conversely, if this identity for \(\delta\) holds, then applying \Cref{eq:fm-recovery} to both sides gives
    \[
    X_{s,t}^\star({\bf x})=X_{u,t}^\star(X_{s,u}^\star({\bf x})).
    \]
    Thus the two identities are equivalent.

\end{proof}

\subsection{Proof of Proposition~\ref{prop:twotimedenoisertraining}}
\label{app:proof_twotimedenoisertraining}

\begin{lavenderbox}
    \twotimedenoisertraining*
\end{lavenderbox}

\begin{proof}
By the two-time denoiser identities \Cref{eq:fm-diagonal} and \Cref{eq:fm-semigroup}, the true two-time denoiser \(\delta\) satisfies
\[
\delta_{t,t}({\bf x})=D_t^\star({\bf x})
\]
and
\[
\delta_{s,t}({\bf x})
=
\gamma \delta_{s,u}({\bf x})
+
(1-\gamma)
\delta_{u,t}\bigl(X_{s,u}^\star({\bf x})\bigr).
\]
Therefore both squared-error terms in \(\mathcal L(\delta)\) are zero, and hence
\[
\mathcal L(\delta)=0.
\]

Conversely, suppose \(\mathcal L(\delta)=0\). Since \(\mathcal L\) is a sum of nonnegative squared norms, each squared norm must vanish almost surely under its corresponding training distribution. The diagonal term gives
\[
\delta_{t,t}({\bf x})=D_t^\star({\bf x})
\]
almost surely. The semigroup term gives
\[
\delta_{s,t}({\bf x})=\mathcal T(\delta)_{s,u,t}({\bf x})
\]
almost surely. Substituting the definition of \(\mathcal T\), this becomes
\[
\delta_{s,t}({\bf x})
=
\gamma \delta_{s,u}({\bf x})
+
(1-\gamma)
\delta_{u,t}\bigl(X_{s,u}^\star({\bf x})\bigr)
\]
almost surely. Thus any zero-loss solution satisfies the diagonal and semigroup consistency identities on the training distributions.
\end{proof}

%% file: sections/appendix_implementation.tex
\section{Experiment Details}
\label{app:impl}

\subsection{Analysis details}
\label{app:analysis_detail}
{
\paragraph{Fixed-point convergence.}
For the convergence analysis in \Cref{fig:refinement-compact-draft}(a), we construct ${\bf x}_t$ from 256 training examples at $t\in\{0.0,0.1,\ldots,0.9\}$ and average the relative L2 distances over both the examples and the ten time points.
We approximate the fixed point using 200 damped Picard iterations~\citep{lauriere2021numerical}:
\[
    {\bf z}^{j+1} = \alpha {\bf z}^{j} + (1-\alpha)\hat{D}_{t}({\bf x}_{t},{\bf z}^{j}),
\]
where ${\bf z}^{0}={\bf 0}$ and $\alpha=0.3$.
The reference ${\bf z}^{\star}$ is the final iterate ${\bf z}^{200}$ from the same run.

\paragraph{Generation quality.}
For the iteration sweep in \Cref{fig:refinement-compact-draft}(b,c), we use warm-start ODE sampling with undamped Picard iterations ($\alpha=0$), setting $\gamma=0$ for ELF and softmax temperature 1.0 for LangFlow.

\paragraph{Warm-start and cold-start sampling.}
For the comparison in \Cref{fig:warm_vs_cold_j8}, we use 32 sampling steps for ELF and 64 for LangFlow, as in \Cref{fig:refinement-compact-draft}(b,c).
In both schemes, we use $\alpha=0$ with one iteration and $\alpha=0.3$ with 100 iterations per flow timestep; warm starts use $\hat{{\bf z}}_{t_0}^0={\bf 0}$ at the first flow timestep.
For the frontiers, we sweep the sampling parameter $\gamma\in\{0.0,0.5,1.0,1.5,2.0\}$ for ELF and the softmax temperature over $\{0.8,0.9,1.0,1.1,1.2\}$ for LangFlow.

\paragraph{Sampling settings.}
We use the PyTorch reproduction of ELF\footnote{\url{https://github.com/Ugness/ELF-pytorch}} and the LangFlow implementation\footnote{\url{https://github.com/nealchen2003/LangFlow}} with their default time grids, solvers, and token decoding.
We set the self-conditioning guidance weight of ELF to $w=1$ and the guidance scale of LangFlow to 1.

}

\subsection{Self-conditioning-free model details}
\label{app:cdeq}

We obtain ELF$^\star$ by distilling the frozen self-conditioned ELF teacher $\hat{D}$ into $D^\star$~\eqref{eq:fp_denoiser} with consistency deep equilibrium (CDEQ) distillation~\citep{lin2026consistency}.

\paragraph{Teacher.} For the teacher $\hat{D}$ in \Cref{subsec:distill}, we train a variant of ELF-B~\citep{hu2026elf} whose T5 text encoder~\citep{raffel2023exploringlimitstransferlearning} is replaced by the last hidden states of a pretrained GPT-2 Large model~\citep{radford2019language}, so that it shares the GPT-2 tokenizer with the baselines; otherwise, it follows the original ELF architecture and hyperparameters.

\paragraph{Architecture.} ELF~\citep{hu2026elf} conditions on the flow time $t$ and the self-conditioning guidance weight $w$ through two sets of four learned prefix tokens, to which embeddings of $t$ and $w$ are added, respectively.
The student is initialized from the teacher, retains its decoder head, and adds a third such set for the consistency time $\tau$ (defined below); these tokens and the $\tau$ embedding are zero-initialized, so the student initially nearly reproduces the teacher.

\paragraph{Training.} Following~\citet{lin2026consistency}, at each training step, we run the teacher's fixed-point iteration~\eqref{eq:sc_map} at the training interpolant from ${\bf z}^0={\bf 0}$ for $20$ Anderson-accelerated iterations (history $3$, mixing $0.9$, and regularization $10^{-4}$) and keep the iterates ${\bf z}^0,\dots,{\bf z}^{20}$, computed without gradients, taking ${\bf z}^{20}$ as ${\bf z}^\star$.
We assign iterate ${\bf z}^j$ a consistency time $\tau_j=\varepsilon+(1-e^{-\rho j})(T-\varepsilon)$ with $\tau_0=\varepsilon$ ($\varepsilon=0.002$, $T=5.0$, $\rho=0.3$), and let $c_{\mathrm{skip}}(\tau)=(\tau-\varepsilon)/(T-\varepsilon)$ and $c_{\mathrm{out}}(\tau)=1-c_{\mathrm{skip}}(\tau)$.
The student predicts a consistency function $g_j=c_{\mathrm{skip}}(\tau_j)\,{\bf z}^j+c_{\mathrm{out}}(\tau_j)\,P_j$, where $P_j$ is a closed-form two-point Anderson combination (mixing $1.0$) of two $\tau$-conditioned student evaluations.
We minimize $\mathcal{L}=\lambda_1|g_j-{\bf z}^\star|^2+(1-\lambda_1)|g_j-\mathsf{sg}(g_{j-1})|^2$ with $\lambda_1=0.8$.
With probability $0.25$, we instead regress $g_0$ directly onto ${\bf z}^\star$.

\paragraph{Inference.} ELF$^\star$ generates with the forward Euler scheme~\eqref{eq:euler}, using one evaluation per flow timestep at $\tau=\tau_0$ with the self-conditioning input set to ${\bf 0}$.

\subsection{Self-conditioned flow map details}
\label{app:flowmap_impl}

\paragraph{Architecture.} The two-time denoiser $\delta_{s,t}$~\eqref{eq:twotimedenoiser} uses the ELF backbone, with its four flow-time prefix tokens split evenly between the target time $t$ and the source time $s$, both embedded by the shared time embedder.

\paragraph{Training.} Following~\citet{lee2026flow}, each training example is diagonal ($s=t$) with probability $0.4$ for offline FMLM$^\star$ or $0.5$ for online FMLM$^\star$, and off-diagonal otherwise; $1/32$ of the off-diagonal examples are fixed at $(s,t)=(0,1)$.
For the semigroup target $\mathcal{T}(\delta)_{s,u,t}$ of~\eqref{eq:fp-distill}, we evaluate $\delta_{s,u}$ and $\delta_{u,t}$ of the current student under stop-gradient, obtain $X^\star_{s,u}$ from $\delta_{s,u}$ via~\eqref{eq:fm-recovery}, and set $u$ to the midpoint of $s$ and $t$ in percentile (CDF) coordinates of the logit-normal time distribution below.
The times follow the teacher's training-time marginal, a logit-normal distribution with $(\mu,\nu)=(-1.5,0.8)$.
We compute the mean-squared error of each term of~\eqref{eq:fp-distill} separately and sum the two with equal weights.

Per example, we draw one self-conditioning guidance weight $w$ log-uniformly from $[0.5,5.0]$, the range used in ELF training~\citep{hu2026elf}, and feed it to every teacher and student evaluation in that example's loss; at inference, $w$ is fixed across all steps.

Offline FMLM$^\star$ obtains $D^\star_t$ from a single evaluation of ELF$^\star$ (\Cref{app:cdeq}); online FMLM$^\star$ instead estimates it from the self-conditioned ELF teacher.
For a budget of $K$ fixed-point iterations (FPIs), we apply ${\bf z}^{j+1}=\hat{D}_t({\bf x},{\bf z}^j)$ for $j=0,\ldots,K-1$ from ${\bf z}^0={\bf 0}$ and use ${\bf z}^K$ as an approximation of $D_t^\star({\bf x})$, so ``\# FPIs'' in \Cref{tab:nrefinemaps} counts all $K$ teacher denoiser evaluations, including the first one with zero conditioning.

ELF$^\star$ and FMLM$^\star$ follow the original ELF hyperparameters and are each trained for 5 epochs with a global batch size of 512 using 8 NVIDIA B200 GPUs.
We use a combination of Muon optimizers (peak learning rate $0.002$, warmup over the first $0.5$ epochs) and use the exponential moving average (EMA) of the weights with decay $0.9999$ for inference.

\paragraph{Inference.} To report the gPPL and entropy for the two- and four-step FMLM$^\star$ (\Cref{tab:onestep,tab:nrefinemaps}), we utilize $\gamma$-sampling \citep{lee2026flow, sabour2026align, kim2024consistency} with $\gamma = 0.75$ and $1.0$, respectively. We sweep this $\gamma$ for the eight-step frontier.

%% file: sections/appendix_results.tex
\section{Additional Results}
\label{app:results}

\subsection{Generation quality across fixed-point iterations}
\label{app:refinement_results}

{
We report the numerical results for \Cref{fig:refinement-compact-draft}(b,c) in \Cref{tab:refinement}, following the setup described in \Cref{app:analysis_detail}.
As the number of fixed-point iterations increases, gPPL decreases for both ELF and LangFlow, whereas entropy remains stable near the data entropy of 5.44 nats.
This supports our finding that taking more steps toward the fixed point improves generation quality.
}

\begin{table}[!ht]
  \centering
  \caption{{Effect of the number of fixed-point iterations on generation quality.}}
  \small
  \begin{tabular}{lccccc}
    \toprule
    ELF (32 steps), \# FPIs & $1$ & $2$ & $3$ & $4$ & $5$ \\
    \midrule
    gPPL $\downarrow$ & 101.19 & 66.50 & 56.67 & 50.25 & 46.31 \\
    entropy              & 5.51 & 5.42 & 5.37 & 5.33 & 5.31 \\
    \toprule
    LangFlow (64 steps), \# FPIs & $1$ & $2$ & $3$ & $4$ & $5$ \\
    \midrule
    gPPL $\downarrow$ & 77.26 & 57.99 & 52.63 & 49.37 & 47.27 \\
    entropy              & 5.49 & 5.42 & 5.40 & 5.38 & 5.36 \\
    \bottomrule
  \end{tabular}
  \label{tab:refinement}
\end{table}

\subsection{Online fixed-point distillation}
\label{app:online_distillation}

We extend the online distillation results in \Cref{subsec:distill}
by varying the number of fixed-point iterations used during training,
with results reported in \Cref{tab:nrefinemaps}.
Increasing the number of FPIs improves gPPL up to 8 FPIs.
With 16 FPIs, however, gPPL improves only slightly for one- and two-step generation and worsens for four-step generation, while training time increases from 23 to 32 hours, suggesting diminishing returns from additional FPIs.

\begin{table}[!ht]
  \centering
  \small
  \setlength{\tabcolsep}{3pt}
  \caption{Comparison of FMLM$^\star$ with online and offline fixed-point distillation. Online distillation with 8 fixed-point iterations achieves competitive generation quality at substantially lower training cost, while additional iterations provide limited benefits.}
  \label{tab:nrefinemaps}
  \begin{tabular}{cc|ccccc|c}
    \toprule
    \multirow{2}{*}{Step} & \multirow{2}{*}{Metric} & \multicolumn{5}{|c|}{\# FPIs for online distillation} & \multirow{2}{*}{Offline} \\
    \cmidrule(lr){3-7}
     & & 1 & 2 & 4 & 8 & 16 & \\
    \midrule
    \multirow{2}{*}{1} & gPPL $\downarrow$ & 176.31 & 141.97 & 123.28 & 117.89 & 116.63 & \textbf{113.91} \\
                               & entropy & 5.47 & 5.42 & 5.41 & 5.39 & 5.37 & 5.37 \\
    \midrule
    \multirow{2}{*}{2} & gPPL $\downarrow$ & 149.51 & 115.25 & 101.27 & 98.96 & 98.61 & \textbf{93.66} \\
                               & entropy & 5.53 & 5.50 & 5.48 & 5.47 & 5.46 & 5.46 \\
    \midrule
    \multirow{2}{*}{4} & gPPL $\downarrow$ & 120.35 & 89.32 & 81.27 & 80.26 & 82.50 & \textbf{73.86} \\
                               & entropy & 5.46 & 5.44 & 5.45 & 5.44 & 5.44 & 5.35 \\
    \midrule
    \multicolumn{2}{c|}{Training time} & 18h & 18h & 20h & 23h & 32h & 82h \\
    \bottomrule
  \end{tabular}
\end{table}

%% file: sections/appendix_samples.tex
\section{Qualitative results}
\label{app:samples}

\subsection{Iterative refinement at a fixed flow state}
\label{app:qualitative_refinement}

{
We qualitatively examine how fixed-point iteration improves the model's predictions, complementing the quantitative analysis in \Cref{subsec:analysis}.
To this end, we monitor the predictions of ELF at a fixed flow state.
We construct an interpolant at $t=0.25$ from a training sequence and Gaussian noise, and compare the decoded predictions over three fixed-point iterations.

As shown in \Cref{fig:qualitative-fpi-t5}, fixed-point iterations improve grammatical consistency and contextual coherence of the prediction.
Specifically, from $j=1$ to $j=2$, the model improves grammatical consistency by changing ``they to change'' to ``they can change'' and ``the counselors and so like'' to ``the counselors and the vendors.''
At $j=3$, the model improves contextual coherence by changing ``their voices and send customers'' to ``their voices and call people'' and ``the vendors'' to ``the staff,'' which fits more naturally with the references to patients and counselors.
These improvements are also supported by the quantitative results in \Cref{fig:refinement-compact-draft}(b,c), where gPPL decreases as the number of fixed-point iterations increases, while entropy stays close to that of the data.
\par}

\begin{figure}[H]
  \centering
  \begingroup
  \newcommand{\fpispan}[1]{\textcolor{darkblue!70!black}{\textbf{#1}}}
  \begin{tcolorbox}[
    colback=keylavender, colframe=black!35,
    colbacktitle=black!50, coltitle=white,
    boxrule=0.5pt, arc=1mm,
    left=7pt, right=7pt, top=6pt, bottom=6pt, boxsep=1pt,
    title={\normalfont\small\textbf{Fixed-point iterations}\hfill
      \normalfont\scriptsize ELF\quad $t=0.25$},
    toptitle=3pt, bottomtitle=3pt
  ]
    \small
    \begin{minipage}[t]{0.07\linewidth}
      \strut $j=1$
    \end{minipage}\hfill
    \begin{minipage}[t]{0.90\linewidth}
      \raggedright\strut
      \ldots{} What? Kunsho's business about helping. Many of them say
      \fpispan{they to change} \fpispan{their phones and send customers}.
      That is certainly the one reason they're catching.
      We get responses from not only \fpispan{the customers}, but also
      \fpispan{the counselors and so like}. \ldots{}
    \end{minipage}
    \par\vspace{7pt}{\color{black!20}\hrule height 0.4pt}\vspace{7pt}
    \begin{minipage}[t]{0.07\linewidth}
      \strut $j=2$
    \end{minipage}\hfill
    \begin{minipage}[t]{0.90\linewidth}
      \raggedright\strut
      \ldots{} What? Kursho's is about helping. Many of them say
      \fpispan{they can change} \fpispan{their voices and send customers}.
      That is certainly the one reason they're catching.
      We get responses from not only \fpispan{the patients}, but also
      \fpispan{the counselors and the vendors}. \ldots{}
    \end{minipage}
    \par\vspace{7pt}{\color{black!20}\hrule height 0.4pt}\vspace{7pt}
    \begin{minipage}[t]{0.07\linewidth}
      \strut $j=3$
    \end{minipage}\hfill
    \begin{minipage}[t]{0.90\linewidth}
      \raggedright\strut
      \ldots{} What is Kurwah's business about helping? Most of them say
      \fpispan{they can change} \fpispan{their voices and call people}.
      That is certainly the big reason they're catching.
      We get responses from not only \fpispan{the patients}, but also
      \fpispan{the counselors and the staff}. \ldots{}
    \end{minipage}
  \end{tcolorbox}
  \endgroup
  \caption{Qualitative samples across fixed-point iterations.
    Fixed-point iterations improve grammatical consistency and contextual coherence.}
  \label{fig:qualitative-fpi-t5}
\end{figure}

\subsection{Few-step generation samples}
\label{app:generation_samples}

We provide generation samples from the FMLM$^\star$ model.
The one-step, two-step, and four-step samples are shown in \Cref{fig:scfmlm_1step_gen1}, \Cref{fig:scfmlm_2step_gen1}, and \Cref{fig:scfmlm_4step_gen1}, respectively.

\clearpage
\begin{figure}[p]
    \centering

    \begin{samplebox}{\normalfont\textbf{Sampling Steps: 1} \hfill \normalfont\small \textcolor{white}{gPPL: 65.54, Entropy: 5.27}}
        \small\linespread{0.9}\selectfont
 for Michael when he was selected to learn last week. He didn't get out of trouble, a tournament game when he came out of day 10 with a win. That's happened in the college setting versus the atmosphere that he played.

That's pretty clear why I always want to like to go five years, what I'm loving doing. It's just why I have how to feel myself every time the year. I'm going enough with the goals, talents and opportunities and feel happy what it meant me.

"As my thing is, I''t want to come across that, or what I calling it up and do the next thing. I need a few things. I feel, if I practice and go back at the year, I can be a major part of a senior team with the most significant effort I have, as a former veteran players.

I'm gonna work out of injury, and the adversity that has some value will be experienced by me. And it's a tough game every year.

My coach is a great player, and I just have to get at the moment. I'm got enough parts I the know that I'm really happy working this year. I have a good season and good day every year. I have to ask people that way and my people to understand it, to meet each other. And so I think this is hopefully for sometime coming back this year.

Here is a signed video playable on a NBA. website, can be in a digital video below.

Like now of the seasons for "LNBA teams are on freeoff for the Cavaliers through late Friday. The Cavaliers are at a superior shooting position. A team who shoot parts of-op per game, High-Footers (for example) and 3pointers (III, -1 (2-1-2) and -4), and out-inging on's pace is -09-30 (48) to -3.

The Clippers are seated in the playoff race and have shot 19.81 percent per year in the league. But this is the first such showing for the Cleveland Celtics this year. Compared to the 105 games this season:

For others, who has been up to the top of the league, have been:

1/AP / Full Press

Batum

At his opener in October NBA season, rookie Guer Turner shot an all-timer of 11.5 points (2311) in first-three points, and seven more points in advance of the 2014 finale.

Barnas is 25th out of 35 a game in the league and his three-mealers go in 34 points, +21.

He also had 15 points to open the game with the finished with 34 (one point). He had connected as a third-seash-point shot and scored a layup (but he didn't look a friendly with his teammates, but
't being a Pac-Star).

Kevin Elin re-recorded his own first plus-pointer since to shoot e, and when he came in 13 minutes, and finished with points on six of his seven Pac-e attempts, all tying the conference record. But he shot the best shot in the White House at that time, which it was his longest score of career.

Payman stays out of The league in pace and offensive/minus.

Last year's 2013 NBA season

Q20 games totaled 13-11 of the players scored in Houston's first playoff campaign as Houston's 3-3 wins. The-Gobel, with Kevin-en, began with the Warriors in Michael Sthoel: He also added 10 points in the first five minutes of, which he was went off to win his fifth point of the game.

Meanwhile, he pushed for Michael O-o to stay up for the game over the next five games. Then also lining up with the he added 15 points. Which was -22 out, just the biggest mark in the's career.

He's shot 20 times in pre-season, and missed 15 points. The Lakers forward Quentin Billard in the second game began to move back. His foot was injured after sidelined with a toe injury. He started T. D'Angelo Robinson in the opener of Game USA.

When he stepped up to San Antonio, it wasn't much to get argue with him. But he has multiple touches in the season, however, he also led up his points for 25 points and 17.

This is the result of injury that followed a three-game incident with an Andrew on St. John with a knee injury. He played 20 minutes in 10 games through the end of 2013, which has seen the game improve this year.

"I have been looking to be there because of a a little challenge," Dr. said.
    \end{samplebox}
    \caption{A sample generated by FMLM$^\star$ with one-step sampling.}
    \label{fig:scfmlm_1step_gen1}
\end{figure}

\clearpage
\begin{figure}[p]
    \centering

    \begin{samplebox}{\normalfont\textbf{Sampling Steps: 2} \hfill \normalfont\small \textcolor{white}{gPPL: 62.97, Entropy: 5.28}}
        \small\linespread{0.9}\selectfont
 Part of gamemaking is an intense field with more funding being forthcoming (I want to announce more next year). All major players will be making different statements on social media sites, as they may say, and we may then find some way to upget players to get that long way. But it might be good time to post more certain important things during the year.

Drake

The game will provide one of our statements:

Obviously, something here is nothing just that we'll never be all the more involved here; there are others that we're not going to be an Easyway and stop it again.

Darth also has a great way of making people right now, so you donate. You all like, share it, and share it to build your following. If you're the follower.

Darth we offer me the incredible opportunity to win a new game on Jump Stones, and I feel activated. The entire campaign happens everyday as you set up and into traction.

It's like a million thousand followers that requires more more than compliments than a smile, obviously. There are still ticket updates, raffles, newbies, campaigns, thousands and dollars to be spent in helping you earn a campaign goal.

Here is some. Set out the bonus points:

In order to bring donations to new players, Jump Stones naturally expands its functionality as the better it offers the best results.

I'm looking proud to give out to Jason [Dundy back in 2004]. I came back to the game at 7 years old, some kids began playing fast shooting games, became gamers, and played games.

In the newsletter I had buttons titled ``Make Good Games Again, I recently completed a senior player'' video that took all of these and notes, getting several people in the audience to take the details of how I changed my mindset.

I couldn't give them the kind of reason I just didn't make a jump game, as a challenge after a few years. But it was a very refreshing realization. When my brother had read the first article about the game, I wanted to deal up with new issues around flight and in a civilized way at Jump Times. But for the large gamer living in the first Zelda game Minecraft, everyone stands for them. The promotion for Derrader's Fortune sets where we need to move forward in time. One simple idea justices you to releasing the first game in months. This strategy basically depends on the desire to make a game that works with PC traffic there that are otherwise casual, along with things you were forced to say about. In my years of working, I hope you saw that the game has been on the rise. The game world has changed since it launched. I have never had the ability to appreciate it, but even something that we've experienced since Jump Time, I also needed a healthy amount of gratitude and Star Tennis.

But as a gamer, you have the opportunity to look back at the hardest to do with it, and your feeling that every aspect of our life has been meaningful to all our players. Before we decided to change our mind, we have never done right for us repay our costs and we needed to take responsibility in our game strategy. Even at the same time that our game would lead the company as a teenager, we couldn't.

If you have a goal, you are not well born warriors. After all, you can fight your world's game, and this game game is hope and protect you.

If you come back to the goal you hope you do, and you don't actually want to pay an expense, make yourself smile by it.

1. Overthinking goals

No. When you are involved it isn't the key. Many people don't have knowledge about how to get them out of the next year. Contrary to the catch tip to remember that implies, there are criteria that practice is difficult to accomplish, and mistakes to pursue. The other decisions will also be influenced on the way a game approaches.

2. As a result, you are ready to struggle to make an game. If you call yourself a novice, you'll gain a challenge and there will always the time to start with it.

3. Be sure to focus your time to your time as possible\ldots{}

Statsicizes how much time your season will take to better see the capacity you require could be compromised.

Humping time will be the approximate amount of time, the risk that this situation will become the game's total time will entail when you get out going.

In other words, you are sure to get the game well during the day as you try your problems.

From the next months, you are likely to end up behind. Publisher updates will be the place and display a various steps along your way to promote H
    \end{samplebox}
    \caption{A sample generated by FMLM$^\star$ with two-step sampling.}
    \label{fig:scfmlm_2step_gen1}
\end{figure}

\clearpage
\begin{figure}[p]
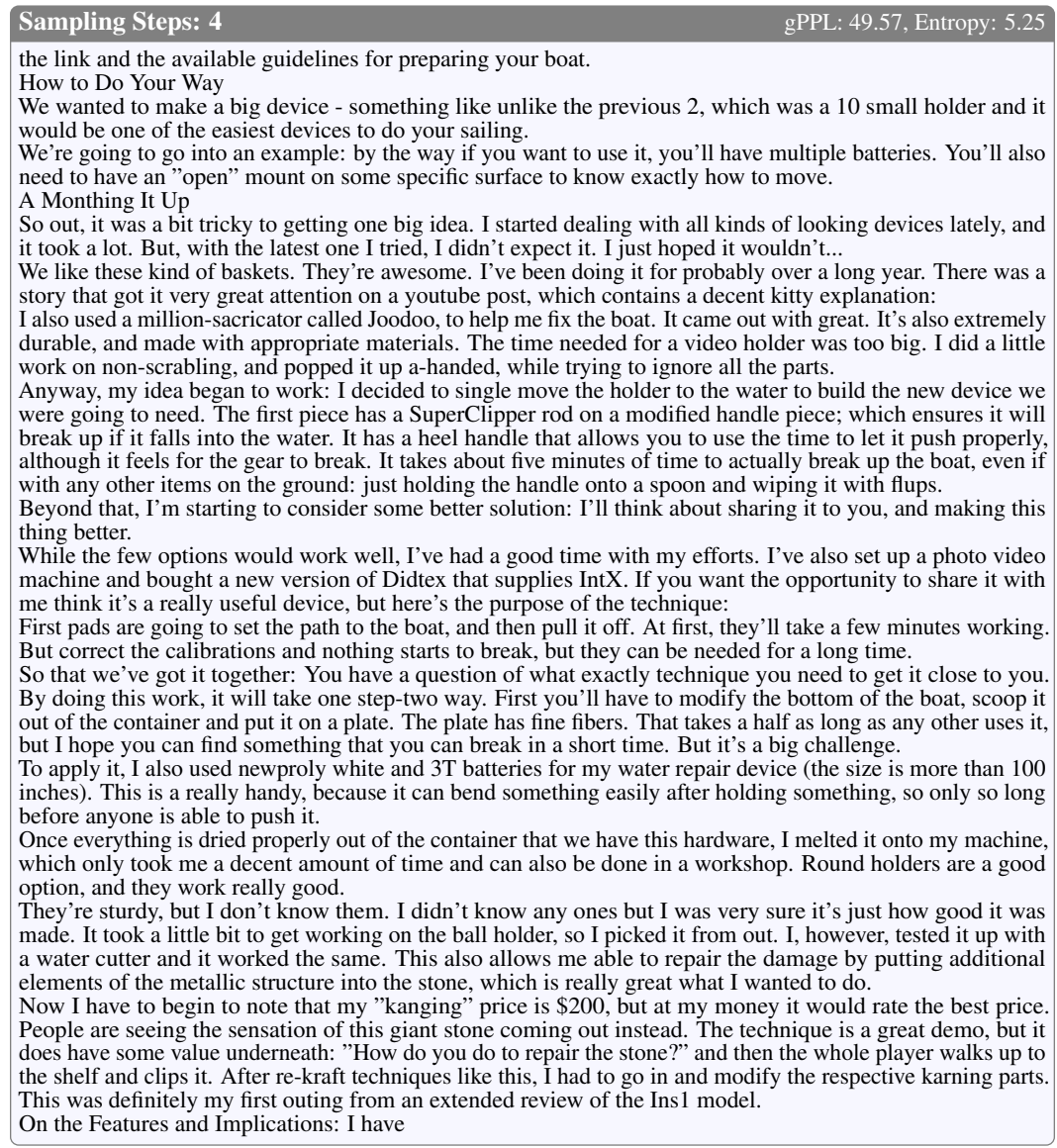

    \centering

    \begin{samplebox}{\normalfont\textbf{Sampling Steps: 4} \hfill \normalfont\small \textcolor{white}{gPPL: 49.57, Entropy: 5.25}}
        \small\linespread{0.9}\selectfont
 the link and the available guidelines for preparing your boat.

How to Do Your Way

We wanted to make a big device - something like unlike the previous 2, which was a 10 small holder and it would be one of the easiest devices to do your sailing.

We're going to go into an example: by the way if you want to use it, you'll have multiple batteries. You'll also need to have an "open" mount on some specific surface to know exactly how to move.

A Monthing It Up

So out, it was a bit tricky to getting one big idea. I started dealing with all kinds of looking devices lately, and it took a lot. But, with the latest one I tried, I didn't expect it. I just hoped it wouldn't...

We like these kind of baskets. They're awesome. I've been doing it for probably over a long year. There was a story that got it very great attention on a youtube post, which contains a decent kitty explanation:

I also used a million-sacricator called Joodoo, to help me fix the boat. It came out with great. It's also extremely durable, and made with appropriate materials. The time needed for a video holder was too big. I did a little work on non-scrabling, and popped it up a-handed, while trying to ignore all the parts.

Anyway, my idea began to work: I decided to single move the holder to the water to build the new device we were going to need. The first piece has a SuperClipper rod on a modified handle piece; which ensures it will break up if it falls into the water. It has a heel handle that allows you to use the time to let it push properly, although it feels for the gear to break. It takes about five minutes of time to actually break up the boat, even if with any other items on the ground: just holding the handle onto a spoon and wiping it with flups.

Beyond that, I'm starting to consider some better solution: I'll think about sharing it to you, and making this thing better.

While the few options would work well, I've had a good time with my efforts. I've also set up a photo video machine and bought a new version of Didtex that supplies IntX. If you want the opportunity to share it with me think it's a really useful device, but here's the purpose of the technique:

First pads are going to set the path to the boat, and then pull it off. At first, they'll take a few minutes working. But correct the calibrations and nothing starts to break, but they can be needed for a long time.

So that we've got it together: You have a question of what exactly technique you need to get it close to you. By doing this work, it will take one step-two way. First you'll have to modify the bottom of the boat, scoop it out of the container and put it on a plate. The plate has fine fibers. That takes a half as long as any other uses it, but I hope you can find something that you can break in a short time. But it's a big challenge.

To apply it, I also used newproly white and 3T batteries for my water repair device (the size is more than 100 inches). This is a really handy, because it can bend something easily after holding something, so only so long before anyone is able to push it.

Once everything is dried properly out of the container that we have this hardware, I melted it onto my machine, which only took me a decent amount of time and can also be done in a workshop. Round holders are a good option, and they work really good.

They're sturdy, but I don't know them. I didn't know any ones but I was very sure it's just how good it was made. It took a little bit to get working on the ball holder, so I picked it from out. I, however, tested it up with a water cutter and it worked the same. This also allows me able to repair the damage by putting additional elements of the metallic structure into the stone, which is really great what I wanted to do.

Now I have to begin to note that my "kanging" price is \$200, but at my money it would rate the best price. People are seeing the sensation of this giant stone coming out instead. The technique is a great demo, but it does have some value underneath: "How do you do to repair the stone?" and then the whole player walks up to the shelf and clips it. After re-kraft techniques like this, I had to go in and modify the respective karning parts. This was definitely my first outing from an extended review of the Ins1 model.

On the Features and Implications: I have
    \end{samplebox}
    \caption{A sample generated by FMLM$^\star$ with four-step sampling.}
    \label{fig:scfmlm_4step_gen1}
\end{figure}